\documentclass[lettersize,journal]{IEEEtran}
\usepackage{amsmath,amsfonts}
\usepackage{algorithmic}
\usepackage{algorithm}
\usepackage{array}
\usepackage{textcomp}
\usepackage{stfloats}
\usepackage{url}
\usepackage{verbatim}
\usepackage{graphicx}
\usepackage{cite}

\usepackage{amsmath,amsfonts,bm,amssymb,bbm}

\def\eqref#1{equation~\ref{#1}}

\def\1{\bm{1}}

\def\vtheta{{\bm{\theta}}}

\def\vd{{\bm{d}}}

\def\vg{{\bm{g}}}

\def\vk{{\bm{k}}}

\def\vm{{\bm{m}}}

\def\vp{{\bm{p}}}

\def\vv{{\bm{v}}}

\def\vx{{\bm{x}}}

\def\vdelta{{\boldsymbol{\delta}}}

\DeclareMathAlphabet{\mathsfit}{\encodingdefault}{\sfdefault}{m}{sl}
\SetMathAlphabet{\mathsfit}{bold}{\encodingdefault}{\sfdefault}{bx}{n}

\def\gG{{\mathcal{G}}}

\def\gL{{\mathcal{L}}}

\def\gS{{\mathcal{S}}}

\def\sR{{\mathbb{R}}}

\usepackage{hyperref}
\usepackage[numbers]{natbib}
\usepackage{url}
\usepackage{subfigure}
\usepackage{hyperref}
\usepackage{multirow}
\usepackage{booktabs}
\usepackage{makecell}
\usepackage{rotating}
\usepackage{wrapfig}
\usepackage{bookmark}

\usepackage{pifont}
\usepackage{multicol}
\usepackage{color}
\usepackage{colortbl}
\usepackage{xcolor}
\usepackage{soul}
\usepackage{mathtools}
\usepackage{amsthm}
\theoremstyle{plain}
\newtheorem{theorem}{Theorem}[section]

\theoremstyle{definition}

\theoremstyle{remark}

\newcommand*\colourcheck[1]{%
  \expandafter\newcommand\csname #1check\endcsname{\textcolor{#1}{\ding{52}}}%
}
\colourcheck{green}

\newcommand*\colourcross[1]{%
  \expandafter\newcommand\csname #1cross\endcsname{\textcolor{#1}{\ding{56}}}%
}
\colourcross{red}

\newcommand{\green}{\textcolor{green}}
\newcommand{\red}{\textcolor{red}}

\begin{document}

\title{Sparse-PGD: A Unified Framework for Sparse Adversarial Perturbations Generation}
\author{\IEEEauthorblockN{Xuyang Zhong\IEEEauthorrefmark{2}, Chen Liu \IEEEauthorrefmark{2}\IEEEauthorrefmark{1}}
        
        \IEEEauthorblockA{\IEEEauthorrefmark{2} City University of Hong Kong\\
xuyang.zhong@my.cityu.edu.hk, chen.liu@cityu.edu.hk}
\thanks{\IEEEauthorrefmark{1} denotes the correspondence author.}
\thanks{Manuscript received 11/2024; revised 09/2025; accepted 11/2025.}}

\markboth{IEEE Transactions on Pattern Analysis and Machine Intelligence}{}


\maketitle

\begin{abstract}
    This work studies sparse adversarial perturbations, including both unstructured and structured ones. We propose a framework based on a white-box PGD-like attack method named Sparse-PGD to effectively and efficiently generate such perturbations. Furthermore, we combine Sparse-PGD with a black-box attack to comprehensively and more reliably evaluate the models' robustness against unstructured and structured sparse adversarial perturbations. Moreover, the efficiency of Sparse-PGD enables us to conduct adversarial training to build robust models against various sparse perturbations. Extensive experiments demonstrate that our proposed attack algorithm exhibits strong performance in different scenarios. More importantly, compared with other robust models, our adversarially trained model demonstrates state-of-the-art robustness against various sparse attacks. Codes are available at \href{https://github.com/CityU-MLO/sPGD}{https://github.com/CityU-MLO/sPGD}.
\end{abstract}


\section{Introduction}

Deep learning has been developing tremendously fast in the last decade. However, it is shown vulnerable to adversarial attacks: imperceivable adversarial perturbations \citep{2013Intriguing, 2016Adversarial} could change the prediction of a model without altering the input's semantic content, which poses great challenges in safety-critical systems.
Among different kinds of adversarial perturbations, the ones bounded by $l_\infty$ or $l_2$ norms are mostly well-studied \citep{2014Explaining, madry2017towards, 2019Theoretically} and benchmarked~\citep{croce2020robustbench}, because these adversarial budgets, i.e., the sets of all allowable perturbations, are convex, which facilitates theoretical analyses and algorithm design.
By contrast, we study \textit{sparse perturbations} in this work, including both \textit{unstructured} and \textit{structured} ones which are bounded by $l_0$ norm and group $l_0$ norm, respectively.
These perturbations are quite common in physical scenarios, including broken pixels in LED screens to fool object detection models and adversarial stickers on road signs to make an auto-driving system fail \citep{papernot2017practical, Akhtar2018ThreatOA, Xu2019AdversarialAA,feng2022graphite,wei2023unified}.

However, constructing such adversarial perturbations is challenging as the corresponding adversarial budget is non-convex.
Therefore, gradient-based methods, such as projected gradient descent (PGD)~\citep{madry2017towards}, usually cannot efficiently obtain a strong adversarial perturbation. 
In this regard, existing methods to generate sparse perturbations~\cite{2018SparseFool, croce2019sparse, su2019one, 2020GreedyFool, rao2020adversarial, croce2022sparse} either cannot control the $l_0$ norm or the group $l_0$ norm of perturbations or have prohibitively high computational complexity, which makes them inapplicable for adversarial training to obtain robust models against sparse perturbations.
The perturbations bounded by $l_1$ norm are the closest scenario to $l_0$ bounded perturbations among convex adversarial budgets defined by an $l_p$ norm.
Nevertheless, adversarial training in this case~\citep{tramer2019adversarial, croce2021mind} still suffers from issues such as slow convergence and instability. \citet{jiang2023towards} demonstrates that these issues arise from non-sparse perturbations bounded by $l_1$ norm.
In other words, $l_1$ adversarial budget cannot guarantee the sparsity of the perturbations. 
Furthermore, existing works investigating structured sparse perturbations \citep{brown2017adversarial, karmon2018lavan, yang2020patch, rao2020adversarial, croce2022sparse} only support generating a single adversarial patch, lacking flexibility and generality. Thus, it is necessary but challenging to develop a unified framework for both unstructured and structured sparse adversarial perturbations.

In this work, we propose a white-box attack named Sparse-PGD (sPGD) to effectively and efficiently generate unstructured and structured sparse perturbations.
For unstructured sparse perturbations, we decompose the perturbation $\bm{\delta}$ as the product of a magnitude tensor $\vp$ and a binary sparse mask $\vm$: $\bm{\delta} = \vp \odot \vm$ , where $\vp$ and $\vm$ determine the magnitudes and the locations of perturbed features, respectively.
Although $\vp$ can be updated by PGD-like methods, it is challenging to directly optimize the binary mask $\vm$ in the discrete space. We thereby introduce an alternative continuous variable $\widetilde{\vm}$ to approximate $\vm$ and update $\widetilde{\vm}$ by gradient-based methods, $\widetilde{\vm}$ is then transformed to $\vm$ by projection to the discrete space.
To further boost the performance, we propose the unprojected gradient of $\vp$ and random reinitialization mechanism.
For structured sparse perturbations, we study group $l_0$ norm and its approximated version based on the group norm proposed in \citep{bach2010structured}. The structured sparse perturbation bounded by approximated group $l_0$ norm can be thus decomposed as $\vdelta=\vp \odot \vm =\vp \odot \min\left(\mathrm{TConv}(\vv, \vk), 1\right)$, where the binary group mask $\vv$ determines the positions of the groups to be perturbed and the customized binary kernel $\vk$ determines the pattern of groups. That is to say, we decompose the structured sparse perturbations into their positions and the patterns. Furthermore, we leverage transposed convolution ($\mathrm{TConv}$) and clipping operations to map the positions of the groups $\vv$ to the positions of perturbed features $\vm$. 
Similar to the unstructured cases, we introduce the continuous $\widetilde{\vv}$ and update it by gradient-based methods.
Ultimately, we manage to transform the problem of optimizing the structured sparse mask $\vm$ into the problem of optimizing the group mask $\vv$ bounded by $l_0$ norm, which can be resolved in the framework of sPGD. The whole pipeline is illustrated in Figure \ref{fig:pipeline}.
On top of sPGD, we propose Sparse-AutoAttack (sAA), which is the ensemble of the white-box sPGD and another black-box sparse attack, for a more comprehensive and reliable evaluation against both unstructured and structured sparse perturbations.
Through extensive experiments, we show that our method exhibits better performance than other attacks.

More importantly, we explore adversarial training to obtain robust models against sparse attacks. 
In this context, the attack method will be called in each mini-batch update, so it should be both effective and efficient.
Compared with existing methods, our proposed sPGD performs much better when using a small number of iterations, making it feasible for adversarial training and its variants~\cite{Zhang2019TheoreticallyPT}. Empirically, models adversarially trained by sPGD demonstrate the strongest robustness against various sparse attacks.

We summarize the contributions of this paper as follows:
\begin{enumerate}
    \item We propose an effective and efficient white-box attack algorithm named Sparse-PGD (sPGD) which can be utilized to generate both unstructured and structured sparse adversarial perturbations.

    \item sPGD achieves the best performance among white-box sparse attacks. We then combine it with a black-box sparse attack to construct Sparse-AutoAttack (sAA) for more comprehensive robustness evaluation against sparse adversarial perturbations.

    \item sPGD achieves much better performance in the regime of limited iterations, it is then adopted for adversarial training. Extensive experiments demonstrate that models adversarially trained by sPGD have significantly stronger robustness against various sparse attacks.
\end{enumerate}

\textbf{Preliminaries:}
We use image classification as an example, although the proposed methods are applicable to any classification model.
Under $l_p$ bounded perturbations, the robust learning aims to solve the following min-max optimization problem.
\begin{equation} \label{eq:adv}
\begin{split}
    \mathop{\min}\limits_{\bm{\theta}}\frac{1}{N}\sum_{i=1}^{N}\mathop{\max}\limits_{\bm{\delta}_i} \mathcal{L}(\bm{\theta}, \vx_i+\bm{\delta}_i),\\
    \mathrm{s.t.}\ ||\bm{\delta}_i||_p \leq \epsilon,\ 0\leq \vx_i+\bm{\delta}_i \leq 1.
\end{split}  
\end{equation}
where $\bm{\theta}$ denotes the parameters of the model and $\gL$ is the loss objective function. $\vx_i \in \mathbb{R}^{h\times w\times c}$ is the input image where $h$, $w$ and $c$ represent the height, width, and number of channels, respectively. $\bm{\delta}_i$ has the same shape as $\vx_i$ and represents the perturbation.
The perturbations are constrained by its $l_p$ norm and the bounding box. In this regard, we use the term \textit{adversarial budget} to represent the set of all allowable perturbations.
Adversarial attacks focus on the inner maximization problem of (\ref{eq:adv}) and aim to find the optimal adversarial perturbation, while adversarial training focuses on the outer minimization problem of (\ref{eq:adv}) and aims to find a robust model parameterized by $\bm{\theta}$.
Due to the high dimensionality and non-convexity of the loss function when training a deep neural network, \cite{Weng2018TowardsFC} has proven that solving the problem (\ref{eq:adv}) is at least NP-complete.

We consider the pixel sparsity for image inputs in this work, which is more meaningful than feature sparsity and consistent with existing works~\citep{croce2019sparse, croce2022sparse}. That is, a pixel is considered perturbed if \textit{any} of its channel is perturbed, and sparse perturbation means few pixels are perturbed.

\section{Related Works}
\textbf{Non-sparse Attacks:} 
The pioneering work \cite{2013Intriguing} finds the adversarial perturbations to fool image classifiers and proposes a method to minimize the $l_2$ norm of such perturbations.
To more efficiently generate adversarial perturbations, the fast gradient sign method (FGSM) \cite{2014Explaining} generates $l_\infty$ perturbation in one step, but its performance is significantly surpassed by the multi-step variants~\cite{kurakin2017adversarial}.
Projected Gradient Descent (PGD) \citep{madry2017towards} further boosts the attack performance by using iterative updating and random initialization. 
Specifically, each iteration of PGD updates the adversarial perturbation $\bm{\delta}$ by:
\begin{equation} \label{eq:pgd}
\bm{\delta} \longleftarrow \Pi_{\gS}(\bm{\delta} + \alpha \cdot s(\nabla_{\bm{\delta}}\mathcal{L}(\vtheta, \vx+\bm{\delta})))    
\end{equation}
where $\gS$ is the adversarial budget, $\alpha$ is the step size, $s: \mathbb{R}^{h\times w\times c} \rightarrow \mathbb{R}^{h\times w\times c}$ selects the steepest ascent direction based on the gradient of the loss $\mathcal{L}$ with respect to the perturbation.
Inspired by the first-order Taylor expansion, \citet{madry2017towards} derives the steepest ascent direction for $l_2$ bounded and $l_\infty$ bounded perturbations to efficiently find strong adversarial examples; SLIDE \citep{tramer2019adversarial} and $l_1$-APGD \citep{croce2021mind} use $k$-coordinate ascent to construct $l_1$ bounded perturbations, which is shown to suffer from the slow convergence~\citep{jiang2023towards}. 

Besides the attacks that have access to the gradient of the input (i.e., white-box attacks), there are black-box attacks that do not have access to model parameters, including the ones based on gradient estimation through finite differences \citep{Bhagoji2018PracticalBA, Ilyas2018BlackboxAA, Ilyas2018PriorCB, Tu2018AutoZOOMAZ, Uesato2018AdversarialRA} and the ones based on evolutionary strategies or random search \citep{Alzantot2018GenAttackPB, Guo2019SimpleBA}. To improve the query efficiency of these attacks, \cite{AlDujaili2019ThereAN, Moon2019ParsimoniousBA, Meunier2019YetAB, Andriushchenko2019SquareAA} generate adversarial perturbation at the corners of the adversarial budget.

To more reliably evaluate the robustness, \citet{croce2020reliable} proposes AutoAttack (AA) which consists of an ensemble of several attack methods, including both black-box and white-box attacks.
\citet{croce2021mind} extends AA to the case of $l_1$ bounded perturbations and proposes AutoAttack-$l_1$ (AA-$l_1$).
Although the $l_1$ bounded perturbations are usually sparse, \citet{jiang2023towards} demonstrates that AA-$l_1$ is able to find non-sparse perturbations that cannot be found by SLIDE to fool the models.
That is to say, $l_1$ bounded adversarial perturbations are not guaranteed to be sparse. We should study perturbations bounded by $l_0$ norm.


\textbf{Sparse Attacks:} For perturbations bounded by $l_0$ norm, directly adopting vanilla PGD as in Eq. (\ref{eq:pgd}) leads to suboptimal performance due to the non-convexity nature of the adversarial budget: $\mathrm{PGD_0}$ \citep{croce2019sparse}, which updates the perturbation by gradient ascent and project it back to the adversarial budget, turns out very likely to trap in the local maxima. Different from $\mathrm{PGD_0}$, CW L0 \citep{carlini2017towards} projects the perturbation onto the feasible set based on the absolute product of gradient and perturbation and adopts a mechanism similar to CW L2 \citep{carlini2017towards} to update the perturbation.
SparseFool \citep{2018SparseFool} and GreedyFool \citep{2020GreedyFool} also generate sparse perturbations, but they do not strictly restrict the $l_0$ norm of perturbations.
If we project their generated perturbations to the desired $l_0$ ball, their performance will drastically drop.
Sparse Adversarial and Interpretable Attack Framework (SAIF) \citep{imtiaz2022saif} is similar to our method in that SAIF also decomposes the $l_0$ perturbation into a magnitude tensor and sparsity mask, but it uses the Frank-Wolfe algorithm \citep{frank1956algorithm} to separately update them. SAIF turns out to get trapped in local minima and shows poor performance on adversarially trained models.
Besides white-box attacks, there are black-box attacks to generate sparse adversarial perturbations, including CornerSearch~\citep{croce2019sparse} and Sparse-RS~\citep{croce2022sparse}.
However, these black-box attacks usually require thousands of queries to find an adversarial example, making it difficult to scale up to large datasets.
In addition to the unstructured adversarial perturbations mentioned above, there are several works discussing structured sparse perturbations, including universal adversarial patch for all data \citep{brown2017adversarial, karmon2018lavan} and some image-specific patches~\citep{yang2020patch, rao2020adversarial, croce2022sparse}. However, these works only support generating a single adversarial patch, which lacks of flexibility and generality.

\textbf{Adversarial Training:} 
Despite the difficulty in obtaining robust deep neural network, adversarial training \citep{madry2017towards, Croce2019MinimallyDA, sehwag2021, rebuffi2021data, 2021Improving, rade2021helperbased, Cui2023DecoupledKD, Wang2023BetterDM} stands out as a reliable and popular approach to do so \citep{Athalye2018ObfuscatedGG, croce2020reliable}.
It generates adversarial examples first and then uses them to optimize model parameters.
Despite effective, adversarial training is time-consuming due to multi-step attacks.
\citet{Shafahi2019AdversarialTF, Zhang2019YouOP, Wong2020FastIB, Sriramanan2021TowardsEA} use weaker but faster one-step attacks to reduce the overhead, but they may suffer from catastrophic overfitting \citep{Kang2021UnderstandingCO}: the model overfits to these weak attacks during training instead of achieving true robustness to various attacks. \citet{Kim2020UnderstandingCO, Andriushchenko2020UnderstandingAI, Golgooni2021ZeroGradM, Jorge2022MakeSN} try to overcome catastrophic overfitting while maintaining efficiency.

Compared with $l_\infty$ and $l_2$ bounded perturbations, adversarial training against $l_1$ bounded perturbations is shown to be even more time-consuming to achieve the optimal performance \citep{croce2021mind}.
In the case of $l_0$ bounded perturbations, PGD$_0$ \cite{croce2019sparse} is adopted for adversarial training. However, models trained by PGD$_0$ exhibit poor robustness against strong sparse attacks.
In this work, we propose an effective and efficient sparse attack that enables us to train a model that is more robust against various sparse attacks than existing methods.


\section{Unstructured Sparse Adversarial Attack} \label{sec:unstruct}
In this section, we introduce Sparse-PGD (sPGD) for generating unstructured sparse perturbations. Its extension that generates structured sparse perturbations is introduced in Sec. \ref{sec:struct}.
Similar to AutoAttack~\cite{croce2020reliable, croce2021mind}, we further combine sPGD with a black-box attack to construct sparse-AutoAttack (sAA) for more comprehensive and reliable robustness evaluation.

\subsection{Sparse-PGD (sPGD)} \label{sec:l0_pgd}
Inspired by SAIF \citep{imtiaz2022saif}, we decompose the sparse perturbation $\bm{\delta}$ into a magnitude tensor $\vp\in\mathbb{R}^{h\times w\times c}$ and a sparsity mask $\vm\in\{0,1\}^{h\times w\times 1}$, i.e., $\bm{\delta} =\vp \odot \vm$. Therefore, the attacker aims to maximize the following loss objective function:
\begin{equation}
    \max_{\|\bm{\delta}\|_0 \leq \epsilon, 0 \leq \vx + \bm{\delta} \leq 1} \gL(\bm{\theta}, \vx + \bm{\delta}) = \max_{\vp \in \gS_\vp, \vm \in \gS_\vm}\mathcal{L}(\bm{\theta}, \vx+\vp \odot \vm). \label{eq:loss_ori}
\end{equation}
The feasible sets for $\vp$ and $\vm$ are $\gS_\vp = \{\vp \in \mathbb{R}^{h\times w\times c} \vert 0 \leq \vx + \vp \leq 1\}$ and $\gS_\vm = \{\vm \in \{0,1\}^{h\times w\times 1} \vert \|\vm\|_0 \leq \epsilon\}$, respectively.
Similar to PGD, sPGD iteratively updates $\vp$ and $\vm$ until finding a successful adversarial example or reaching the maximum iteration number. 

\textbf{Update Magnitude Tensor $\vp$:} The magnitude tensor $\vp$ is only constrained by the input domain. In the case of images, the input is bounded between $0$ and $1$. 
Note that the constraints on $\vp$ are elementwise and similar to those of $l_\infty$ bounded perturbations.
Therefore, instead of greedy or random search \citep{croce2019sparse,croce2022sparse}, we utilize PGD in the $l_\infty$ case, i.e., use the sign of the gradients, to optimize $\vp$ as demonstrated by Eq. (\ref{eq:delta}) below, with $\alpha$ being the step size.
\begin{equation}
    \vp \longleftarrow \Pi_{\gS_{\vp}}\left(\vp + \alpha \cdot \mathtt{sign}(\nabla_{\vp} \gL(\bm{\theta}, \vx + \vp \odot \vm))\right), \label{eq:delta}
\end{equation}

\textbf{Update Sparsity Mask $\vm$:} 
The sparsity mask $\vm$ is binary and constrained by its $l_0$ norm.
Directly optimizing the discrete variable $\vm$ is challenging, so we update its continuous alternative $\widetilde{\vm} \in \mathbb{R}^{h\times w\times 1}$ and project $\widetilde{\vm}$ to the feasible set $\gS_\vm$ to obtain $\vm$ before multiplying it with the magnitude tensor $\vp$ to obtain the sparse perturbation $\bm{\delta}$.
Specifically, $\widetilde{\vm}$ is updated by gradient ascent. Projecting $\widetilde{\vm}$ to the feasible set $\gS_{\vm}$ is to set the $\epsilon$-largest elements in $\widetilde{\vm}$ to $1$ and the rest to $0$.
In addition, we adopt the sigmoid function to normalize the elements of $\widetilde{\vm}$ before projection.

Mathematically, the update rules for $\widetilde{\vm}$ and $\vm$ are demonstrated as follows:
\begin{align}
    \widetilde{\vm} &\longleftarrow \widetilde{\vm} + \beta \cdot \nabla_{\widetilde{\vm}}\gL / \|\nabla_{\widetilde{\vm}} \gL\|_2,  \label{eq:m_3} \\
    \vm &\longleftarrow \Pi_{\gS_{\vm}}(\sigma(\widetilde{\vm})) \label{eq:m_project}
\end{align}
where $\beta$ is the step size for updating the sparsity mask's continuous alternative $\widetilde{\vm}$, $\sigma(\cdot)$ denotes the sigmoid function. 
Furthermore, to prevent the magnitude of $\widetilde{\vm}$ from becoming explosively large, we do not update $\widetilde{\vm}$ when $||\nabla_{\widetilde{\vm}}\gL||_2 < \gamma$, which indicates that $\widetilde{\vm}$ is located in the saturation zone of sigmoid function.
The gradient $\nabla_{\widetilde{\vm}} \gL$ is calculated at the point $\bm{\delta} = \vp \odot \Pi_{\gS_{\vm}}(\sigma(\widetilde{\vm}))$, where the loss function is not always differentiable. We demonstrate how to estimate the update direction in the next part.

\textbf{Backward Function:} Based on Eq.~(\ref{eq:loss_ori}), we can calculate the gradient of the magnitude tensor $\vp$ by $\nabla_\vp \gL = \nabla_{\bm{\delta}} \gL(\bm{\theta}, \vx + \bm{\delta}) \odot \vm$ and use $g_{\vp}$ to represent this gradient for notation simplicity. 
At most, $\epsilon$ non-zero elements are in the mask $\vm$, so $g_\vp$ is sparse and has at most $\epsilon$ non-zero elements. That is to say, we update at most $\epsilon$ elements of the magnitude tensor $\vp$ based on the gradient $g_\vp$. Like coordinate descent, this may result in suboptimal performance since most elements of $\vp$ are unchanged in each iterative update.
To tackle this problem, we discard the projection to the binary set $\gS_{\vm}$ when calculating the gradient and use the \textit{unprojected gradient} $\widetilde{g}_{\vp}$ to update $\vp$.
Based on Eq.~(\ref{eq:m_project}), we have $\widetilde{g}_{\vp} = \nabla_{\bm{\delta}} \gL(\bm{\theta}, \vx + \bm{\delta}) \odot \sigma(\widetilde{\vm})$.
The idea of the unprojected gradient is inspired by training pruned neural networks and lottery ticket hypothesis~\citep{frankle2018the, ramanujan2020s, fu2021drawing, liu2022robust}. All these methods train importance scores to prune the model parameters but update the importance scores based on the whole network instead of the pruned sub-network to prevent the sparse update, which leads to suboptimal performance.

In practice, the performance of using $g_{\vp}$ and $\widetilde{g}_{\vp}$ to optimize $\vp$ is complementary. The sparse gradient $g_{\vp}$ is consistent with the forward propagation and is thus better at exploitation. By contrast, the unprojected gradient $\widetilde{g}_{\vp}$ updates the $\vp$ by a dense tensor and is thus better at exploration.
In view of this, we set up an ensemble of attacks with both gradients to balance exploration and exploitation.

When calculating the gradient of the continuous alternative $\widetilde{\vm}$, we have $\frac{\partial \gL}{\partial \widetilde{\vm}} = \frac{\partial \gL(\theta, \vx + \bm{\delta})}{\partial \bm{\delta}} \odot \vp \odot \frac{\partial \Pi_{\gS_{\vm}} (\sigma(\widetilde{\vm}))}{\partial \widetilde{\vm}}$. Since the projection to the set $\gS_{\vm}$ is not always differentiable, we discard the projection operator and use the approximation $\frac{\partial \Pi_{\gS_{\vm}} (\sigma(\widetilde{\vm}))}{\partial \widetilde{\vm}} \simeq \sigma'(\widetilde{\vm})$ to calculate the gradient.

\begin{algorithm}[tb]
\caption{Sparse-PGD for $l_0$ Bounded Perturbations}\label{alg:main}
\begin{algorithmic}[1]
    \STATE {\bfseries Input:} Clean image: $\vx \in [0, 1]^{h\times w\times c}$; Model parameters: $\vtheta$; Max iteration number: $T$; Tolerance: $t$; $l_0$ budget: $\epsilon$; Step size: $\alpha$, $\beta$; Small constant: $\gamma=2\times10^{-8}$
    \STATE Random initialize $\vp$ and $\widetilde{\vm}$
     \STATE $\vm = \Pi_{\gS_{\vm}}(\sigma(\widetilde{\vm}))$
    \FOR {$i=0, 1, ..., T-1$} 
        \STATE Calculate the loss $\mathcal{L}(\vtheta, \vx + \vp\odot \vm)$
        \IF{unprojected}  
        \STATE $g_{\vp} = \nabla_{\bm{\delta}} \gL \odot \sigma(\widetilde{\vm})$ \quad\quad\quad\quad\quad\quad
        \COMMENT{$\bm{\delta} = \vp\odot \vm$}
        \ELSE
        \STATE $g_{\vp} = \nabla_{\bm{\delta}} \gL \odot \vm$ 
        \ENDIF
        \STATE $g_{\widetilde{\vm}} = \nabla_{\bm{\delta}} \gL \odot \vp \odot \sigma'(\widetilde{\vm})$
        \STATE $\vp = \Pi_{\gS_{\vp}}(\vp + \alpha \cdot \mathtt{sign}(g_{\vp}))$
         
        \STATE $\vd = g_{\widetilde{\vm}} / \|g_{\widetilde{\vm}}\|_2$ \textbf{if} $||g_{\widetilde{\vm}}||_2 \geq \gamma$ \textbf{else} $0$
        \STATE $\vm_{old},\ \widetilde{\vm} = \vm,\ \widetilde{\vm} + \beta \cdot \vd$
        \STATE $\vm = \Pi_{\gS_{\vm}}(\sigma(\widetilde{\vm}))$
        \IF{attack succeeds} 
        \STATE break
        \ENDIF
        \IF{$||\vm - \vm_{old}||_0 \leq 0$ for $t$ consecutive iters}
        \STATE Random initialize $\widetilde{\vm}$
        \ENDIF
    \ENDFOR
    \STATE {\bfseries Output:} Perturbation: $\bm{\delta} = \vp\odot \vm$
\end{algorithmic}
\end{algorithm}
\textbf{Random Reinitialization:} Due to the projection to the set $\gS_{\vm}$ in Eq. (\ref{eq:m_project}), the sparsity mask $\vm$ changes only when the relative magnitude ordering of the continuous alternative $\widetilde{\vm}$ changes. In other words, slight changes in $\widetilde{\vm}$ usually mean no change in $\vm$.
As a result, $\vm$ usually gets trapped in a local maximum.
To solve this problem, we propose a random reinitialization mechanism.
Specifically, when the attack fails, i.e., the model still gives the correct prediction, and the current sparsity mask $\vm$ remains unchanged for three consecutive iterations, the continuous alternative $\widetilde{\vm}$ will be randomly reinitialized for better exploration.

To summarize, we provide the pseudo-code of sparse PGD (sPGD) in Algorithm \ref{alg:main}. 
SAIF~\cite{imtiaz2022saif} also decomposes the perturbation $\bm{\delta}$ into a magnitude tensor $\vp$ and a mask $\vm$, but uses a different update rule: it uses Frank-Wolfe to update both $\vp$ and $\vm$. By contrast, we introduce the continuous alternative $\widetilde{\vm}$ of $\vm$ and use gradient ascent to update $\vp$ and  $\widetilde{\vm}$.
Moreover, we include unprojected gradient and random reinitialization techniques in Algorithm~\ref{alg:main} to further enhance the performance.

\subsection{Sparse-AutoAttack (sAA)} \label{sec:saa}

AutoAttack (AA) \citep{croce2020reliable} is an ensemble of four diverse attacks for a standardized parameters-free and reliable evaluation of robustness against $l_\infty$ and $l_2$ attacks. \citet{croce2021mind} extends AutoAttack to $l_1$ bounded perturbations. In this work, we propose sparse-AutoAttack (sAA), which is also a parameter-free ensemble of both black-box and white-box attacks for comprehensive robustness evaluation against $l_0$ bounded perturbations. It can be used in a plug-and-play manner.
However, different from the $l_\infty$, $l_2$ and $l_1$ cases, the adaptive step size, momentum and difference of logits ratio (DLR) loss function do not improve the performance in the $l_0$ case, so they are not adopted in sAA.
In addition, compared with targeted attacks, sPGD turns out stronger when using a larger query budget in the untargeted settings given the same total number of back-propagations.
As a result, we only include the untargeted sPGD with cross-entropy loss and constant step sizes in sAA.
Specifically, we run sPGD twice for two different backward functions: one denoted as sPGD$_{\mathrm{proj}}$ uses the sparse gradient $g_{\vp}$, and the other denoted as sPGD$_{\mathrm{unproj}}$ uses the unprojected gradient $\widetilde{g}_{\vp}$ as described in Section \ref{sec:l0_pgd}.
As for the black-box attack, we adopt the strong black-box attack Sparse-RS \citep{croce2022sparse}, which can generate $l_0$ bounded perturbations.
We run each version of sPGD and Sparse-RS for $10000$ iterations, respectively.
We use cascade evaluation to improve the efficiency. Concretely, suppose we find a successful adversarial perturbation by one attack for one instance. Then, we will consider the model non-robust in this instance and the same instance will not be further evaluated by other attacks.
Based on the efficiency and the attack success rate, the attacks in sAA are sorted in the order of \textbf{sPGD$_{\mathrm{unproj}}$}, \textbf{sPGD$_{\mathrm{proj}}$} and \textbf{Sparse-RS}.

\section{Structured Sparse Adversarial Attack} \label{sec:struct}
In this section, we extend Sparse-PGD (sPGD) and Sparse-AutoAttack (sAA) to generate structured sparse perturbations. 

\subsection{Formulation of Structured Sparsity}
Given an input $\vx\in\mathbb{R}^d$, we can partition its $d$ features into several groups that may overlap and define structured sparsity based on them. 
These groups can represent pixels of a row, a column, a patch, or a particular pattern.
Without the loss of generality, we use $\{1, 2, ..., d\}$ as the indices of $d$ features of the input and $N$ set of indices $\gG=\{G_j\}_{j = 1}^N$ to represent the groups.
For any perturbation, we define its \textit{group $l_0$ norm based on groups $\gG$} as \textit{the minimal number} of groups needed to cover its non-zero components.
Mathematically, \citet{bach2010structured} proposed a convex envelope of group $l_0$ norm as follows:
\begin{equation} \label{eq:gn_convex}
    \Omega(\vx)
    = \sum_{k=1}^d |x_{\pi_k}|\left[ F(\{\pi_1,...,\pi_k\}) - F(\{\pi_1,...,\pi_{k-1}\})\right].
\end{equation}
where $F$ is a submodular function defined on the subsets of $V=\{1,2,...,d\}$ and $\{\pi_i\}_{i = 1}^d$ is a permutation of $\{1,2,...,d\}$.
More specifically, $F(A)$ indicates the minimal number of groups from $\gG$ to cover the set $A$ and $\{\pi_i\}_{i = 1}^d$ indicates the components of $|\vx|$ in the decreasing order, i.e., $|x_{\pi_1}|\geq...\geq|x_{\pi_d}|\geq0$.
Compared with the convex envelope $\Omega$, the group $l_0$ norm $\Omega_0$ is defined as follows:
\begin{equation} \label{eq:gn}
\begin{aligned}
    \Omega_0(\vx) &= \sum_{k=1}^d \mathbbm{1}(|x_{\pi_k}|)\left[ F(\{\pi_1,...,\pi_k\}) - F(\{\pi_1,...,\pi_{k-1}\})\right] \\
    &= F(\{\pi_1, ..., \pi_{d'}\})\ \ \mathrm{where}\ x_{\pi_{d'}} \neq 0\ \mathrm{and}\ x_{\pi_{d' + 1}} = 0.
\end{aligned}
\end{equation}
Similar to the difference between $l_1$ norm and $l_0$ norm,  we replace $|x_{\pi_k}|$ with $\mathbbm{1}(|x_{\pi_k}|)$ in Eq. (\ref{eq:gn}) to indicate the number of groups with non-zero entities, where $\mathbbm{1}$ is an indicator function. 
The second equality in Eq. (\ref{eq:gn}) is based on the decreasing order of $\{|x_{\pi_i}|\}_{i = 1}^d$ and is a more straightforward definition of the group $l_0$ norm.
Like the $l_0$ norm, the adversarial budgets based on the group $l_0$ norm are not convex, either.

In general, it is difficult to decide the minimum number of groups to cover non-zero elements, i.e., to calculate the function $F$.
Therefore, it is challenging to directly apply sPGD to the group $l_0$ norm constraints. To address this issue, we can approximate $\Omega_0$ by a tight surrogate that facilitates the optimization by sPGD.
In this regard, we introduce a binary group mask $\vv=[v_1,...,v_{N}] \in \{0, 1\}^N$ to indicate whether a specific group is chosen to be perturbed. 
Mathematically, $\forall i$, if $\vx_i \neq 0$, then $\exists j \in \{1, 2, ..., N\}$, $v_j = 1$ and $i \in G_j$.
In this context, we propose the following approximation of the group $l_0$ norm as the surrogate:
\begin{equation} \label{eq:gn_approx}
     \Omega_0'(\vx,\vv)
    = \sum_{i=1}^N \mathbbm{1}(\|\vx_{G_i}\|_p)\cdot v_i,
\end{equation}
where $\vx_{G_i}$ is the subvector of $\vx$ on the indices in $G_i$, and $p \in [0, +\infty]$. 
Compared with $\Omega_0$ in Eq. (\ref{eq:gn}), $\Omega'_0$ in Eq. (\ref{eq:gn_approx}) does not require the number of groups in $\gG$ to cover non-zero elements in $\vx$ to be minimum.
Instead, $\Omega'_0$ calculates the number of perturbed groups in $\gG$ selected by the binary vector $\vv$.
In addition, when $\mathbbm{1}(\|\vx_{G_i}\|_p)=1$ for all $i\in\{j~|~v_j=1\}$, the approximated group $l_0$ norm $\Omega_0'(\vx,\vv)$ is equivalent to the $l_0$ norm of the group mask $\vv$, i.e., $\Omega_0'(\vx,\vv) = \|\vv\|_0$. In summary, we have the following theorem:
\begin{theorem} \label{theorem:gn}
Given an input $\vx\in\mathbb{R}^d$, a set of groups $\mathcal{G}=\{G_j\}_{j = 1}^N$ and any vector $\vv$ satisfying the constraint in the definition of $\Omega'_0$, then we have
$\Omega_0(\vx) \leq \Omega_0'(\vx,\vv) \leq \|\vv\|_0$. 
\end{theorem}

\begin{proof}
The first inequality is based on the minimum optimality condition of the function $F$ in Eq. (\ref{eq:gn}). Since $\Omega_0(\vx)$ indicates the minimal number of groups needed to cover the indices of non-zero elements in $\vx$, we have at least the same number of terms in Eq. (\ref{eq:gn_approx}) where $\vx_{G_i} \neq 0$ and $v_i = 1$. Therefore, we have $\Omega_0(\vx) \leq \Omega'_0(\vx)$.
For the second inequality, $\vv$ is a binary vector, i.e., $\forall i$, $\vv_i \in \{0, 1\}$, so $\Omega_0'(\vx,\vv) \leq \|\vv\|_0$ is clear.
\end{proof}

Theorem \ref{theorem:gn} indicates that the number of perturbed groups determined by $\vv$ can be larger than the minimal number of groups to cover the perturbed features $\vx$, i.e., $\Omega_0'(\vx,\vv)$ can be larger than $\Omega_0(\vx)$.
The gap between $\Omega_0'(\vx,\vv)$ and $\Omega_0(\vx)$ stem from the potential overlap among groups in $\gG$.
For instance, if we generate $1\times 3$ perturbations, for a $\vx$ with the shape of $1\times 5$ and all of its entities are perturbed, the corresponding $\Omega_0(\vx)=2$, while the possible $\|\vv\|_0$ can be $2$ or $3$, i.e., $\vv=[1,0,1]$ or $[1,1,1]$. Thus, we have $\Omega_0(\vx) \leq \Omega_0'(\vx,\vv) \leq \|\vv\|_0$.
However, we demonstrate that $\Omega_0'(\vdelta,\vv)$ is quite close to $\Omega_0(\vdelta)$ in practice. To corroborate this, we calculate the ratio between the group $l_0$ norm and approximated group $l_0$ norm, i.e., $\Omega_0(\vdelta)/\Omega_0'(\vdelta,\vv)$ of perturbations under different sparsity levels $\epsilon$ defined by $\Omega_0'(\vdelta,\vv) \leq \epsilon$. As illustrated in Figure \ref{fig:ratio}, when applying patch perturbations, we can observe that the ratio $\Omega_0(\vdelta)/\Omega_0'(\vdelta,\vv)$ increases as the ratio between the maximum number of perturbed pixels and the total number of pixels decreases, because the probability of overlapping accordingly declines.
Specifically, when $\epsilon$ is smaller than $5$, the ratios of perturbations on both CIFAR-10 and ImageNet-100 are larger than $0.95$. Notably, on ImageNet-100, the $l_0$ norm of $5$ non-overlapping $10\times10$ patches reaches $500$, which is very large for a sparse perturbation used in practice, so we can conclude based on the simulation results that $\Omega_0'$ is a good approximation of $\Omega_0$ across datasets with different resolutions.
\begin{figure}[t]
    \centering
    \includegraphics[width=0.85\linewidth]{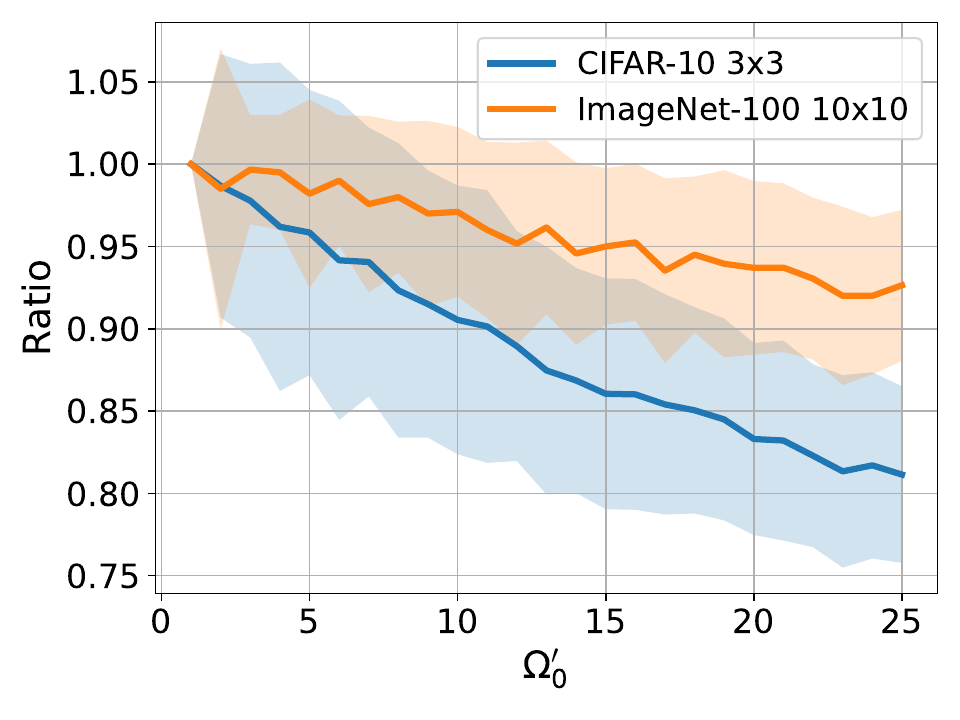}
    \vspace{-1em}
    \caption{The ratio between the group $l_0$ norm and the approximated group $l_0$ norm. The x-axis is the approximated group $l_0$ norm ranging from 1 to 25, and the y-axis is the ratio. We plot the ratios of $3\times 3$ patch perturbations on CIFAR-10 ($32 \times 32$) and $10 \times 10$ patch perturbations on ImageNet-100 ($224 \times 224$), respectively. The results are calculated on 100 samples. The solid line and shadow denote the mean value and standard deviation, respectively.}
    \vspace{-1em}
    \label{fig:ratio}
\end{figure}

\subsection{sPGD for Structured Sparse Perturbations}
\begin{figure*}
    \centering
    \includegraphics[width=0.75\linewidth]{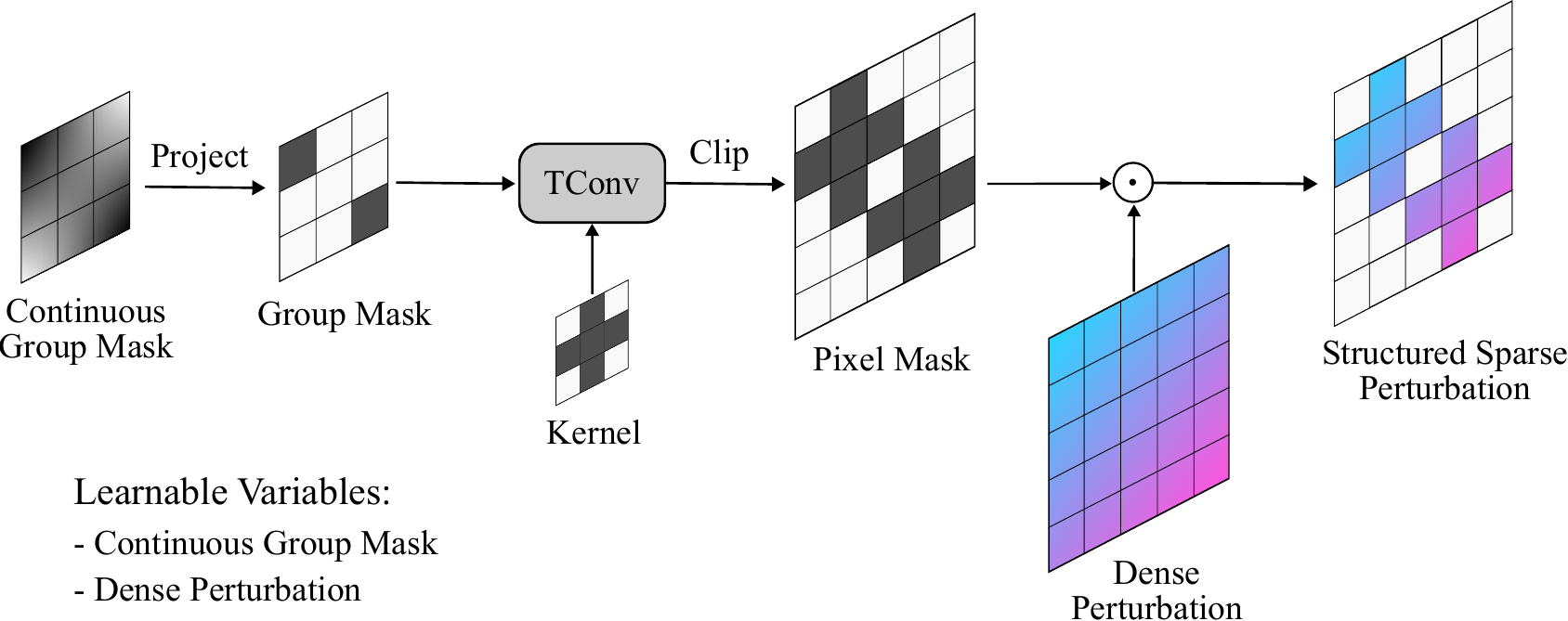}
    \caption{Pipeline of sPGD for structured sparse perturbations. The continuous group mask $\widetilde{\vv}$ is first projected to get the binary group mask $\vv$ to ensure $\|\vv\|_0 \leq \epsilon$, which is similar to Eq. (\ref{eq:m_project}). Given the kernel $\vk \in \{0, 1\}^{r \times r}$ with a customized pattern, we can transform $\vv$ to the pixel mask $\vm$ using transposed convolution and clipping (see Eq. (\ref{eq:map})). Finally, we element-wisely multiply $\vm$ with the dense perturbation $\vp$ to obtain the structured sparse perturbation $\vdelta$. Note that the continuous group mask $\widetilde{\vv}$ and the dense perturbation $\vp$ are learnable.}
    \label{fig:pipeline}
\end{figure*}
Based on the approximated group $l_0$ norm $\Omega_0'$ proposed in Eq. (\ref{eq:gn_approx}), we extend Sparse-PGD (sPGD) to generate structured sparse perturbations. Similarly, we decompose the perturbation $\vdelta$ into a magnitude tensor $\vp$ and a binary pixel mask $\vm$. Like unstructured cases, the magnitude tensor $\vp$ is updated using $l_\infty$ PGD, as in Eq. (\ref{eq:delta}). 

Based on the definition of structured sparsity and $\Omega'_0$ defined in Eq. (\ref{eq:gn_approx}), if a specific group $G_i$ is chosen to be perturbed, then we let $v_i = 1$, otherwise $v_i$ will be set $0$. 
In this regard, finding the optimal $\vm$ s.t. $\Omega'_0(\vm,\vv)\leq \epsilon$ is equivalent to finding the optimal group mask $\vv$ s.t. $\|\vv\|_0\leq \epsilon$, which can be resolved by sPGD. Similar to the unstructured cases, we optimize the continuous alternative $\widetilde{\vv}$ rather than optimizing a binary $\vv$ directly.
In the following, we will elaborate on the mapping from $\vv$ to $\vm$ using different examples.  

\textbf{Column / Row:} If we aim to perturb at most $\epsilon$ columns in an image $\vx\in[0,1]^{h\times w\times c}$. Since we can partition $\vx$ into $w$ groups (i.e., columns), $\vv$ has the size of $1\times w$. Thus,  $\vm$ can obtained through expanding $\vv$ from $1\times w$ to $h\times w$. As for row, the only difference is that the size of $\vv$ is $h\times 1$. Note that, there is no overlap between any two columns or rows, so $\Omega'_0(\vm,\vv)\equiv\Omega_0(\vm)$ in this case.

\textbf{Patch / Any Pattern:}
Without loss of generality, we assume that the perturbations are contained to be localized in $r\times r$ patches. To avoid wasting budget due to the potential overlap at the border or corner of images, we let $\vv \in \{0,1\}^{(h-r+1)\times (w-r+1)}$, and we can derive the mapping operation from $\vv$ to $\vm$ as follows:
\begin{equation} \label{eq:map}
    \vm = \min\left(\mathrm{TConv}(\vv, \vk, s), 1\right),
\end{equation}
where $\mathrm{TConv}$ is the transposed convolution operation, $\vk=\mathbf{1}^{r\times r}$ is the kernel with all-one entities, $s=1$ denotes the stride. The clipping operation is to ensure the output value assigned to $\vm$ is binary.
Moreover, we can customize the kernel $\vk\in\{0,1\}^{r\times r}$ in Eq.~(\ref{eq:map}) to make the perturbations localized in groups with any desired patterns, such as hearts, stars, letters and so on.
Notably, when $\vk=[1]$, i.e., a $1 \times 1$ matrix with a sole value $1$, the resulting perturbations are degraded to unstructured sparse ones; when $\vk=\mathbf{1}^{h\times1}$ or $\mathbf{1}^{1\times w}$, specific columns or rows are perturbed. The pipeline of sPGD for structured sparse perturbations in the general case is illustrated in Figure \ref{fig:pipeline}.

In addition to the forward mapping operation defined in Eq. (\ref{eq:map}), we provide its backward function as follows:
\begin{equation}
    \vg_{\vv} = \mathrm{Conv}(\vg_{\vm}, \vk, s),
\end{equation}
where $\mathrm{Conv}$ is the convolution operation, and $\vg_{\vm}$ is the gradient of $\vm$. We neglect the clipping operation when calculating the gradient of $\vv$ since it is not always differentiable.

The extension of sPGD for structured sparse perturbations offers a unified framework to generate sparse adversarial perturbations. We provide its pseudo-code in Algorithm \ref{alg:structure}.

\begin{algorithm}[tb]
\caption{Sparse-PGD for Structured Sparse Perturbations}\label{alg:structure}
\begin{algorithmic}[1]
    \STATE {\bfseries Input:} Clean image: $\vx \in [0, 1]^{h\times w\times c}$; Model parameters: $\vtheta$; Customized pattern kernel: $\vk \in \{0, 1\}^{r \times r}$; Stride: $s = 1$; Max iteration number: $T$; Tolerance: $t$; group $l_0$ budget: $\epsilon$; Step size: $\alpha$, $\beta$; Small constant: $\gamma=2\times10^{-8}$
    \STATE Random initialize $\vp$ and $\widetilde{\vv}$
    \STATE $\vv = \Pi_{\gS_{\vv}}(\sigma(\widetilde{\vv}))$ where $\gS_{\vv} = \{\vv \vert \|\vv\|_0 \leq \epsilon\}$.
    \STATE $\vm = \min\left(\mathrm{TConv}(\vv, \vk, s), 1\right)$
    \FOR {$i=0, 1, ..., T-1$} 
        \STATE Calculate the loss $\mathcal{L}(\vtheta, \vx + \vp\odot \vm)$
        \IF{unprojected}  
        \STATE $g_{\vp} = \nabla_{\vdelta} \gL \odot \min\left(\mathrm{TConv}(\sigma(\widetilde{\vv}), \vk, s), 1\right)$
        
        \ELSE
        \STATE $g_{\vp} = \nabla_{\bm{\delta}} \gL \odot \vm$ 
        \ENDIF
        \STATE $g_{\widetilde{\vv}} = \mathrm{Conv}(\nabla_{\vm}\gL, \vk, s) \odot \sigma'(\widetilde{\vv})$
        \STATE $\vp = \Pi_{\gS_{\vp}}(\vp + \alpha \cdot \mathtt{sign}(g_{\vp}))$
         
        \STATE $\vd = g_{\widetilde{\vv}} / \|g_{\widetilde{\vv}}\|_2$ \textbf{if} $||g_{\widetilde{\vv}}||_2 \geq \gamma$ \textbf{else} $0$
        \STATE $\vv_{old},\ \widetilde{\vv} = \vv,\ \widetilde{\vv} + \beta \cdot \vd$ 
        \STATE $\vv = \Pi_{\gS_{\vv}}(\sigma(\widetilde{\vv}))$
        \STATE $\vm = \min\left(\mathrm{TConv}(\vv, \vk, s), 1\right)$
        \IF{attack succeeds} 
        \STATE break
        \ENDIF
        \IF{$||\vv - \vv_{old}||_0 \leq 0$ for $t$ consecutive iters}
        \STATE Random initialize $\widetilde{\vv}$
        \ENDIF
    \ENDFOR
    \STATE {\bfseries Output:} Perturbation: $\vdelta = \vp\odot \vm$
\end{algorithmic}
\end{algorithm}

\subsection{sAA for Structured Sparse Perturbations}
Similar to Section \ref{sec:saa}, we also combine sPGD$_{\mathrm{unproj}}$, sPGD$_{\mathrm{proj}}$ and Sparse-RS to build Sparse-AutoAttack (sAA) for reliable evaluation of robustness against structured sparse perturbations. 
We extended the original version of Sparse-RS so that it can generate structured sparse perturbations with any customized patterns.
Specifically, we just apply a pattern mask to the patches generated by Sparse-RS.


\section{Adversarial Training}

In addition to adversarial attacks, we explore adversarial training to build robust models against sparse perturbations.
In the framework of adversarial training, the attack is used to generate adversarial perturbation in each training iteration, so the attack algorithm should not be too computationally expensive.
In this regard, we run the untargeted sPGD (Algorithm \ref{alg:main}) for $20$ iterations to generate sparse adversarial perturbations during training.
We incorporate sPGD in the framework of vanilla adversarial training~\citep{madry2017towards} and TRADES~\citep{Zhang2019TheoreticallyPT} and name corresponding methods \textbf{sAT} and \textbf{sTRADES}, respectively. Note that we use sAT and sTRADES as two examples of applying sPGD to adversarial training since sPGD can be incorporated into any other adversarial training variant as well. To accommodate the scenario of adversarial training, we make the following modifications to sPGD.

\textbf{Random Backward Function:} Since the sparse gradient and the unprojected gradient as described in Section~\ref{sec:l0_pgd} induce different exploration-exploitation trade-offs, we randomly select one of them to generate adversarial perturbations for each mini-batch when using sPGD to generate adversarial perturbations. Compared with mixing these two backward functions together, as in sAA, random backward function does not introduce computational overhead.

\textbf{Multi-$\epsilon$ Strategy:} Inspired by $l_1$-APGD \citep{croce2021mind} and Fast-EG-$l_1$ \citep{jiang2023towards}, multi-$\epsilon$ strategy is adopted to strengthen the robustness of model. 
That is, we use a larger sparsity threshold, i.e., $\epsilon$ in Algorithm~\ref{alg:main}, in the training phase than in the test phase.

\textbf{Higher Tolerance for Reinitialization:} The default tolerance for reinitialization in sPGD is $3$ iterations, which introduces strong stochasticity.
However, in the realm of adversarial training, we have a limited number of iterations. As a result, the attacker should focus more on the exploitation ability to ensure the strength of the generated adversarial perturbations.
While stochasticity introduced by frequent reinitialization hurts exploitation, we find a higher tolerance for reinitialization improves the performance. In practice, we set the tolerance to $10$ iterations in adversarial training.

\section{Experiments}
In this section, we conduct extensive experiments to compare our attack methods with baselines in evaluating the robustness of various models against $l_0$ bounded and structured sparse perturbations.
Besides the effectiveness with an abundant query budget, we also study the efficiency of our methods.
Our results demonstrate that our sPGD performs best among white-box attacks. With limited iterations, sPGD achieves significantly better performance than existing methods. Therefore, sAA, consisting of the best white-box and black-box attacks, has the best attack success rate. sPGD, due to its efficiency, is utilized for adversarial training to obtain the best robust models against $l_0$ bounded and structured sparse adversarial perturbations.
To further demonstrate the efficacy of our method, we evaluate the transferability of sPGD. The results show that sPGD has a high transfer success rate, making it applicable in practical scenarios.
Moreover, to exhibit the versatility of our method, we evaluate sPGD in objection detection and segmentation tasks.
The adversarial examples generated by our methods are presented in Sec. \ref{sec:vis}. 
In addition, we conduct ablation studies for analysis in Appendix C. 
Implementation details are deferred to Appendix A.

\begin{table*}[!t]
\setlength\tabcolsep{5pt}
\vspace{-1em}
\small
    \centering
    \caption{\small Robust accuracy of various models on different sparse attacks. The models are trained on \textbf{CIFAR-10}. Note that we report results of Sparse-RS (RS) with fine-tuned hyperparameters, which outperforms its original version in \citet{croce2022sparse}. CornerSearch (CS) is evaluated on 1000 samples due to its high computational complexity.}
\begin{minipage}{0.48\textwidth}
    \subtable[\small $\epsilon=$ 15]{
    \small
    \resizebox{\textwidth}{!}{
    \small
 \begin{tabular}{l c|c|c c |c c c c c |c}
        \toprule[1.5pt]
        \multirow{2}{*}{Model} & \multirow{2}{*}{Network} & \multirow{2}{*}{Clean} & \multicolumn{2}{c|}{Black} & \multicolumn{5}{c|}{White} & \multirow{2}{*}{\textbf{sAA}} \\
          & & & CS & RS  & SF & PGD$_0$ & SAIF & \textbf{sPGD$_{\mathrm{p}}$}& \textbf{sPGD$_{\mathrm{u}}$}& \\
         \midrule[1pt]
         Vanilla & RN-18 & 93.9  & 1.6 &  0.0  & 25.3& 2.1& 12.0 &  0.0 & 0.0 & \textbf{0.0}\\
         
         \midrule[1pt]
         \multicolumn{10}{l}{$l_\infty$-adv. trained, $\epsilon=8/255$}\\
         \midrule[0.5pt]
         GD & PRN-18 & 87.4 &30.5 & 12.2 & 61.1 & 36.0 & 51.3 & 17.1 & 24.3 & \textbf{11.3}\\
         PORT & RN-18 & 84.6 & 30.8& 15.2 &62.1 & 31.4 & 52.1 &  17.4 & 23.0 & \textbf{13.0}\\
         DKL& WRN-28 & 92.2 & 35.3& 13.2 & 62.5 & 41.2& 52.3& 18.4 & 24.9 & \textbf{12.1}\\
         DM& WRN-28 & 92.4 &34.8& 12.6 & 57.9& 38.5& 49.4& 17.9 & 24.0 & \textbf{11.6}\\

         \midrule[1pt]
         \multicolumn{10}{l}{$l_2$-adv. trained, $\epsilon=0.5$}\\
         \midrule[0.5pt]
         HAT & PRN-18 & 90.6 & 38.9 &23.5  & 65.3  & 35.4 & 60.2 & 19.0 & 18.6 & \textbf{16.6} \\
         PORT & RN-18 & 89.8 & 36.8 &20.6 & 64.3  & 30.6& 59.7 & 16.0 & 15.5 & \textbf{13.8}\\
         DM& WRN-28 & 95.2 &48.5 & 27.7 & 68.2 & 47.5 & 70.9& 25.0 & 26.7 & \textbf{22.1}\\
         FDA & WRN-28 & 91.8 & 47.8& 31.1 & 71.8&40.1 & 68.2& 28.0 & 31.4 & \textbf{25.5}\\
         
         \midrule[1pt]
         \multicolumn{10}{l}{$l_1$-adv. trained, $\epsilon=12$}\\
         \midrule[0.5pt]
         $l_1$-APGD & PRN-18 & 80.7 & 41.3 & 36.5 & 70.3  & 50.5& 62.3& 30.4 & 31.3 & \textbf{29.0}\\
         Fast-EG-$l_1$ & PRN-18 & 76.2 & 40.7 & 34.8 & 64.9 & 46.7 & 56.9& 29.6 &30.1 & \textbf{28.0}\\

         \midrule[1pt]
         \multicolumn{10}{l}{$l_0$-adv. trained, $\epsilon=15$}\\
         \midrule[0.5pt]
         PGD$_0$-A & PRN-18 & 83.7 & 17.5 & 6.1 & 73.7& 62.9 & 60.5 & 19.4 & 27.5& \textbf{5.6} \\
         PGD$_0$-T & PRN-18 & 90.5 & 19.5 & 7.2 & 85.5 & 63.6 & 69.8& 31.4 & 41.2& \textbf{7.1} \\
         \textbf{sAT} & PRN-18  & 80.9 & 46.0 & 37.6 & 77.1 & 74.1 & 72.3& 71.2 & 70.3& \textbf{37.6}  \\
         \textbf{sTRADES} & PRN-18 & 90.3 &71.7  & 63.7 & 89.5 &  88.1 & 86.5& 85.9 & 83.8& \textbf{63.7} \\
         \bottomrule[1.5pt]
    \end{tabular}
    }}
\end{minipage}
\begin{minipage}{0.48\textwidth}
    \subtable[\small $\epsilon=$ 20]{
    \small
    \resizebox{\textwidth}{!}{
    \small
 \begin{tabular}{l c|c|c c |c c c c c| c}
        \toprule[1.5pt]
          \multirow{2}{*}{Model} & \multirow{2}{*}{Network} & \multirow{2}{*}{Clean} & \multicolumn{2}{c|}{Black} & \multicolumn{5}{c|}{White} & \multirow{2}{*}{\textbf{sAA}} \\
          & & & CS & RS  & SF & PGD$_0$ & SAIF & \textbf{sPGD$_{\mathrm{p}}$}& \textbf{sPGD$_{\mathrm{u}}$}& \\
         \midrule[1pt]
         Vanilla & RN-18 & 93.9 & 1.2 & 0.0 & 17.5 & 0.4&3.2 & 0.0 & 0.0&\textbf{0.0}\\
         
         \midrule[1pt]
         \multicolumn{11}{l}{$l_\infty$-adv. trained, $\epsilon=$ 8/255}\\
         \midrule[0.5pt]
         GD & PRN-18 & 87.4 & 26.7&  6.1 &52.6& 25.2 & 40.4& 9.0 & 15.6 &\textbf{5.3}\\
         PORT & RN-18 & 84.6 &  27.8 & 8.5 &54.5 & 21.4 & 42.7 & 9.1 & 14.6 &\textbf{6.7}\\
         DKL& WRN-28 & 92.2 & 33.1& 7.0 &54.0& 29.3& 41.1& 9.9 & 15.8 &\textbf{6.1}\\
         DM& WRN-28 & 92.4 & 32.6& 6.7 &49.4& 26.9& 38.5& 9.9 & 15.1 &\textbf{5.9}\\

         \midrule[1pt]
         \multicolumn{11}{l}{$l_2$-adv. trained, $\epsilon=$ 0.5}\\
         \midrule[0.5pt]
         HAT & PRN-18 & 90.6  & 34.5 & 12.7 & 56.3 & 22.5 & 49.5 &9.1 & 8.5  &\textbf{7.2}\\
         PORT & RN-18 & 89.8 &  30.4& 10.5 &55.0& 17.2& 48.0 & 6.3 & 5.8 &\textbf{4.9}\\
         DM& WRN-28 & 95.2 & 43.3& 14.9 &59.2& 31.8 &59.6 & 13.5 &  12.0 &\textbf{10.2}\\
         FDA & WRN-28 & 91.8 & 43.8& 18.8 & 64.2& 25.5 &57.3 & 15.8 & 19.2 &\textbf{14.1}\\
         
         \midrule[1pt]
         \multicolumn{11}{l}{$l_1$-adv. trained, $\epsilon=$ 12}\\
         \midrule[0.5pt]
         $l_1$-APGD & PRN-18 & 80.7 & 32.3 & 25.0 & 65.4 & 39.8 & 55.6& 17.9 & {18.8} & \textbf{16.9}\\
         Fast-EG-$l_1$ & PRN-18 & 76.2  & 35.0&  24.6 & 60.8 & 37.1& 50.0&  18.1 & {18.6} & \textbf{16.8} \\

         \midrule[1pt]
         \multicolumn{11}{l}{$l_0$-adv. trained, $\epsilon=$ 20}\\
         \midrule[0.5pt]
         PGD$_0$-A & PRN-18 & 77.5 & 16.5 & 2.9 & 62.8& 56.0 & 47.9& 9.9 &21.6 &\textbf{2.4}\\
         PGD$_0$-T & PRN-18 &  90.0 & 24.1 & 4.9 &85.1 & 61.1 & 67.9& 27.3 & 37.9 &\textbf{4.5}\\
        
         \textbf{sAT} & PRN-18 & 84.5 & 52.1 & 36.2&81.2 &  78.0& 76.6 &  75.9  & 75.3  &\textbf{36.2}\\
         
         \textbf{sTRADES} & PRN-18 & 89.8 & 69.9 & 61.8&88.3 & 86.1 & 84.9 & 84.6 & 81.7 & \textbf{61.7}\\
         
         \bottomrule[1.5pt]
    \end{tabular}
    }}
\end{minipage}
    \label{tab:20}
    \vspace*{-1em}
\end{table*}

\subsection{Robustness against Unstructured Sparse Perturbations} \label{sec:evaluation}
First, we compare our proposed sPGD, including sPGD$_{\mathrm{proj}}$ (sPGD$_{\mathrm{p}}$) and sPGD$_{\mathrm{unproj}}$ (sPGD$_{\mathrm{u}}$) as defined in Section~\ref{sec:saa}, and sAA with existing white-box and black-box attacks that generate $l_0$ bounded sparse perturbations.
We evaluate different attack methods based on the models trained on CIFAR-10 \citep{krizhevsky2009learning} and report the robust accuracy with $\epsilon=15,~20$ on the whole test set in Table \ref{tab:20}. Additionally, the results on ImageNet-100 \citep{imagenet_cvpr09} and a real-world traffic sign dataset GTSRB \citep{stallkamp2012man} are reported in Table \ref{tab:more}. Note that the image sizes in GTSRB vary from $15\times 15$ to $250\times 250$. For convenience, we resize them to $224\times 224$ and use the same model architecture as in ImageNet-100. Furthermore, only the training set of GTSRB has annotations, we manually split the original training set into a test set containing $1000$ instances and a new training set containing the rest data.
In Appendix B, we report more results on CIFAR-10 and CIFAR-100 \citep{krizhevsky2009learning} in Table VIII, to demonstrate the efficacy of our methods. 

\textbf{Models:} We select various models to comprehensively evaluate their robustness against $l_0$ bounded perturbations. As a baseline, we train a ResNet-18 (RN-18) \citep{he2016deep} model on clean inputs. For adversarially trained models, we select competitive models that are publicly available, including those trained against $l_\infty$, $l_2$ and $l_1$ bounded perturbations. For the $l_\infty$ case, we include adversarial training with the generated data (GD) \citep{2021Improving}, the proxy distributions (PORT) \citep{sehwag2021}, the decoupled KL divergence loss (DKL) \citep{Cui2023DecoupledKD} and strong diffusion models (DM) \citep{Wang2023BetterDM}. For the $l_2$ case, we include adversarial training with the proxy distributions (PORT) \citep{sehwag2021}, strong diffusion models (DM) \citep{Wang2023BetterDM}, helper examples (HAT) \citep{rade2021helperbased} and strong data augmentations (FDA) \citep{rebuffi2021data}. The $l_1$ case is less explored in the literature, so we only include $l_1$-APGD adversarial training \citep{croce2021mind} and the efficient Fast-EG-$l_1$ \citep{jiang2023towards} for comparison.
The network architecture used in these baselines is either ResNet-18 (RN-18), PreActResNet-18 (PRN-18) \citep{he2016deep} or WideResNet-28-10 (WRN-28) \citep{zagoruyko2016wide}. 
For the $l_0$ case, we evaluate PGD$_0$ \citep{croce2019sparse} in vanilla adversarial training (PGD$_0$-A) and TRADES (PGD$_0$-T) using the same hyper-parameter settings as in \citet{croce2019sparse}. Since other white-box sparse attacks present trivial performance in adversarial training, we do not include their results. Finally, we use our proposed sPGD in vanilla adversarial training (sAT) and TRADES (sTRADES) to obtain PRN-18 models to compare with these baselines. 

\textbf{Attacks:} We compare our methods with various existing black-box and white-box attacks that generate $l_0$ bounded perturbations. The black-box attacks include CornerSearch (CS) \citep{croce2019sparse} and Sparse-RS (RS) \citep{croce2022sparse}. The white-box attacks include SparseFool (SF) \citep{2018SparseFool}, PGD$_0$ \citep{croce2019sparse} and Sparse Adversarial and Interpretable Attack Framework (SAIF) \citep{imtiaz2022saif}. 
The implementation details of each attack are deferred to Appendix A.
Specifically, to exploit the strength of these attacks in reasonable running time, we run all these attacks for either $10000$ iterations or the number of iterations where their performances converge.
Note that, the number of iterations for all these attacks are no smaller than their default settings.
In addition, we report the results of RS with fine-tuned hyperparameters, which outperforms its default settings in \citep{croce2022sparse}. 
Finally, we report the robust accuracy under CS attack based on only $1000$ random test instances due to its prohibitively high computational complexity. 

Based on the results in Table~\ref{tab:20} and Table~\ref{tab:more}, we can find that SF attack, PGD$_0$ attack and SAIF attack perform significantly worse than our methods for all the models studied.
That is, our proposed sPGD always performs the best among white-box attacks.
Among black-box attacks, CS attack can achieve competitive performance, but it runs dozens of times longer than our method does. Therefore, we focus on comparing our method with RS attack.
For $l_1$ and $l_2$ models, our proposed sPGD significantly outperforms RS attack. By contrast, RS attack outperforms sPGD for $l_\infty$ and $l_0$ models. 
This gradient masking phenomenon is, in fact, prevalent across sparse attacks. Given sufficient iterations, RS outperforms all other existing white-box attacks for $l_\infty$ and $l_0$ models. Nevertheless, among white-box attacks, sPGD exhibits the least susceptibility to gradient masking and has the best performance.
The occurrence of gradient masking in the context of $l_0$ bounded perturbations can be attributed to the non-convex nature of adversarial budgets. In practice, the perturbation updates often significantly deviate from the direction of the gradients because of the projection to the non-convex set.
Similar to AA in the $l_1$, $l_2$ and $l_\infty$ cases, sAA consists of both white-box and black-box attacks for comprehensive robustness evaluation. It achieves the best performance in all cases in Table~\ref{tab:20} and Table~\ref{tab:more} by a considerable margin.
\begin{table*}[!t]
\setlength\tabcolsep{5pt}
\centering
\caption{\small Robust accuracy of various models on different sparse attacks. Our sAT model is trained with $\epsilon=1200$. \textbf{(a)} Results on \textbf{ImageNet-100} \citep{imagenet_cvpr09}. \textbf{(b)} Results on \textbf{GTSRB} \citep{stallkamp2012man}. Note that the results of Sparse-RS (RS) with tuned hyperparameters are reported. All models are RN-34, and are evaluated on $500$ samples. The results of SparseFool (SF) and PGD$_0$ are included due to their poor performance, and CornerSearch (CS) is not evaluated here due to its high computational complexity, i.e. nearly 1 week on one GPU for each run.} \label{tab:more}
\vspace{-1em}
\subtable[\textbf{ImageNet-100}, $\epsilon=$ 200]{
\small
    \centering
    \begin{tabular}{l|c|c| c c c |c}
        \toprule[1.5pt]
        \multirow{2}{*}{Model} & \multirow{2}{*}{Clean} & \multicolumn{1}{c|}{Black} & \multicolumn{3}{c|}{White} & \multirow{2}{*}{\textbf{sAA}} \\
          & & RS & SAIF & \textbf{sPGD$_{\mathrm{p}}$}& \textbf{sPGD$_{\mathrm{u}}$}& \\
         \midrule[1pt]
         Vanilla & 83.0 & 0.2  & 0.6 & 0.2 & 0.4 & \textbf{0.0}\\
         \midrule[1pt]
         \multicolumn{7}{l}{$l_0$-adv. trained, $\epsilon=200$}\\
         \midrule[0.5pt]
         PGD$_0$-A & 76.0 &  6.8 & 11.0 & 1.8 & 18.8& \textbf{1.8}\\
         \textbf{sAT} & 86.2 & 61.4 & 69.0 & 78.0 &77.8 & \textbf{61.2}\\
         \bottomrule[1.5pt]
    \end{tabular}
}
\subtable[\textbf{GTSRB}, $\epsilon=$ 600]{
    \centering
    \small

    \begin{tabular}{l|c|c| c c c |c}
        \toprule[1.5pt]
        \multirow{2}{*}{Model}& \multirow{2}{*}{Clean} & \multicolumn{1}{c|}{Black} & \multicolumn{3}{c|}{White} & \multirow{2}{*}{\textbf{sAA}} \\
          & & RS  & SAIF & \textbf{sPGD$_{\mathrm{p}}$}& \textbf{sPGD$_{\mathrm{u}}$}& \\
         \midrule[1pt]
         Vanilla & 99.9 &  18.0 & 9.5 & \textbf{0.3} & \textbf{0.3} & \textbf{0.3}\\
         \midrule[1pt]
         \multicolumn{7}{l}{$l_0$-adv. trained, $\epsilon=600$}\\
         \midrule[0.5pt]
         PGD$_0$-A & 99.8 &37.6 & 13.8 & \textbf{0.0}& \textbf{0.0}& \textbf{0.0}\\
         \textbf{sAT} & 99.8 & 88.4 & 96.2 & 88.6 & 96.2 & \textbf{85.4}\\
         \bottomrule[1.5pt]
    \end{tabular}
}
    \vspace*{-1em}
\end{table*}

In the case of $l_0$ adversarial training, the models are adversarially trained against sparse attacks. However, Figure \ref{fig:iter_k} illustrates that the performance of RS attack on standard, $l_1$-, $l_2$- and $l_\infty$-trained models drastically deteriorates with limited iterations (e.g., smaller than $100$), so RS is not suitable for adversarial training where we need to generate strong adversarial perturbations in limited number iterations.
Empirical evidence suggests that employing RS with $20$ iterations for adversarial training, i.e., the same number of iterations as in other methods, yields trivial performance, so it is not included in Table~\ref{tab:20} or Table~\ref{tab:more} for comparison.
In addition, models trained by PGD$_0$-A and PGD$_0$-T, which generate $l_0$ bounded perturbations, exhibit poor robustness to various attack methods, especially sAA.
By contrast, the models trained by sAT and sTRADES show the strongest robustness, indicated by the comprehensive sAA method and all other attack methods. 
Compared with sAT, sTRADES achieves better performance in both robustness and accuracy. 
Finally, models trained with $l_1$ bounded perturbations are the most robust ones among existing non-$l_0$ training methods. It could be attributed to the fact that $l_1$ norm is the tightest convex relaxation of $l_0$ norm \citep{Bittar2021BestCL}. From a qualitative perspective, $l_1$ attacks also generate relatively sparse perturbations \citep{jiang2023towards}, which makes the corresponding model robust to sparse perturbations to some degree. 

Our results indicate sPGD and RS can complement each other.
Therefore, sAA, an AutoAttack-style attack that ensembles both attacks achieves the state-of-the-art performance on all models.
It is designed to have a similar computational complexity to AutoAttack in $l_\infty$, $l_2$ and $l_1$ cases.

\subsection{Robustness against Structured Sparse Perturbations}
\begin{table*}[!t]
\setlength\tabcolsep{5pt}
\centering
\caption{\small Robust accuracy of various models on different structured sparse attacks.
\textbf{(a)} Results on \textbf{CIFAR-10}, all models are PRN-18.
\textbf{(b)} Results on \textbf{ImageNet-100}, the test set contains $500$ samples, all models are RN-34. Note that the evaluated group sparsity levels are approximately equivalent to the $l_0$ sparsity levels used in Table \ref{tab:20} and \ref{tab:more}, and the Sparse-RS (RS) used here is the proposed extended version.} \label{tab:struct}
\vspace{-1em}
\begin{minipage}{0.51\textwidth}
    \subtable[\small \textbf{CIFAR-10}]{
    \small
    \resizebox{\textwidth}{!}{
    \begin{tabular}{l|c|c |c c c |c}
        \toprule[1.5pt]
        \multirow{2}{*}{Model} & \multirow{2}{*}{Clean} & Black & \multicolumn{3}{c|}{White} & \multirow{2}{*}{\textbf{sAA}} \\
          & & RS  & LOAP & \textbf{sPGD$_{\mathrm{p}}$}& \textbf{sPGD$_{\mathrm{u}}$}& \\
         \midrule[1pt]
         \multicolumn{7}{l}{row, $\epsilon=1$}\\
         \midrule[0.5pt]
         Vanilla & 93.9 & 22.5 & - & 1.0 & 1.3 & \textbf{1.0} \\
         \textbf{sTRADES-$l_0$} & 89.8 &83.1& - &  54.6 & 65.0 & \textbf{53.9}\\
         \textbf{sTRADES-row} & 89.3 & 85.9& - &  77.3 & 83.7 & \textbf{77.3} \\
         \midrule[1pt]
         \multicolumn{7}{l}{$3\times 3$ patch, $\epsilon=2$}\\
         \midrule[0.5pt]
         Vanilla & 93.9 & 6.7 & 4.5 & 1.9& 6.6 & \textbf{1.7} \\
         LOAP+TRADES & 92.8 & 75.6& 66.4 &  47.4 & 71.7 & \textbf{47.1} \\
         \textbf{sTRADES-$l_0$} & 89.8 &65.5& 60.0&  46.9 & 67.5 & \textbf{46.7}\\
         \textbf{sTRADES-p3x3} & 87.9 & 77.5 & 81.4&   72.8 & 81.7 & \textbf{72.3} \\
         \midrule[1pt]
         \multicolumn{7}{l}{$5\times 5$ patch, $\epsilon=1$}\\
         \midrule[0.5pt]
         Vanilla & 93.9 & 7.5 & 2.3 & 1.7 & 3.2 & \textbf{1.7} \\
         LOAP+TRADES & 92.4 & 76.9& 61.7 & 37.0 & 60.8 & \textbf{36.0} \\
         \textbf{sTRADES-$l_0$} & 89.8 &56.4& 35.0 &  26.7 & 44.3 & \textbf{26.7}\\
         \textbf{sTRADES-p5x5} & 89.5 & 72.9& 73.1 &  56.0 & 73.3 & \textbf{55.9} \\
         \bottomrule[1.5pt]
    \end{tabular}
    }
}
\end{minipage}
\begin{minipage}{0.48\textwidth}
    \subtable[\small \textbf{ImageNet-100}]{
    \small
    \resizebox{\textwidth}{!}{
    \begin{tabular}{l|c|c |c c c |c}
        \toprule[1.5pt]
        \multirow{2}{*}{Model} & \multirow{2}{*}{Clean} & Black & \multicolumn{3}{c|}{White-Box} & \multirow{2}{*}{\textbf{sAA}} \\
          & & RS  & LOAP & \textbf{sPGD$_{\mathrm{p}}$}& \textbf{sPGD$_{\mathrm{u}}$}& \\
         \midrule[1pt]
         \multicolumn{7}{l}{row, $\epsilon=1$}\\
         \midrule[0.5pt]
         Vanilla & 83.0 & 54.8 & - & 5.2 & 8.0 & \textbf{5.2} \\
         \textbf{sAT-$l_0$} & 86.2 &78.6& - & 65.8 & 73.8 & \textbf{65.8}\\
         \textbf{sAT-row} & 81.0 & 78.2& - &  76.6 & 78.4 & \textbf{76.6}\\
         \midrule[1pt]
         \multicolumn{7}{l}{$10\times 10$ patch, $\epsilon=2$}\\
         \midrule[0.5pt]
         Vanilla &83.0 & 13.0& 4.6 & 4.6& 6.4 & \textbf{4.0} \\
         LOAP+AT &83.8&57.0 & 72.4 & 0.0 &0.4 & \textbf{0.0}\\
         \textbf{sAT-$l_0$} & 86.2 &27.8& 10.8&  10.6 & 33.0 & \textbf{10.0}\\
         \textbf{sAT-p10x10} &  81.6 & 64.0 & 62.8& 35.8 & 73.2 & \textbf{35.8}\\
         \midrule[1pt]
         \multicolumn{7}{l}{$14\times 14$ patch, $\epsilon=1$}\\
         \midrule[0.5pt]
         Vanilla &  83.0& 7.5 & 2.3 & 1.7 & 12.0 & \textbf{1.7} \\
         LOAP+AT &  83.4& 53.0& 20.8 &  0.0 & 530& \textbf{0.0}\\
         \textbf{sAT-$l_0$} & 86.2 & 28.3& \textbf{10.2} &  15.4 & 34.8 & 14.6\\
         \textbf{sAT-p14x14} & 80.0 &57.2 & 73.0 &41.2  & 71.2& \textbf{39.6}\\
         \bottomrule[1.5pt]
    \end{tabular}
    }
    }
\end{minipage}
\end{table*}
Apart from unstructured sparse perturbations, we evaluate the effectiveness of our method in Algorithm~\ref{alg:structure} in generating structured sparse perturbations in this subsection. Given the paucity of available methods for comparison, we focus on comparing our method with the extended Sparse-RS (RS) and LOAP \citep{rao2020adversarial}, which is a white-box approach generating adversarial patches, on CIFAR-10 and ImageNet-100. For a comprehensive evaluation, we include the results of different types of structured sparse perturbations, e.g., row and patches with different group sparsity level and different sizes. It should be noted that the unstructured $l_0$ norms of the structured perturbations studied here are approximately the same as the perturbations in Table \ref{tab:20} and \ref{tab:more}. In addition, we evaluate the attacks on vanilla models, models trained with unstructured $l_0$ bounded perturbations, and those trained with the specific structured sparse perturbations.

The results in Table \ref{tab:struct} indicate that the proposed sAA also achieves the state-of-the-art performance in generating structured sparse perturbations in almost all cases. Furthermore, the models trained with our methods exhibit the strongest robustness against structured sparse perturbations. As anticipated, models trained on specific structured sparse adversarial examples significantly outperform those trained on unstructured sparse adversarial examples. This highlights the inherent limitation of adversarially trained models in maintaining robustness against unseen types of perturbations during training.

\subsection{Adversarial Watermarks}
Although only a few pixels are perturbed in sparse attacks, such perturbations are still perceptible due to their unconstrained magnitude. In this subsection, we additionally constrain the $l_\infty$ norm of magnitude $\vp$, meaning that the feasible set for $vp$ is rewritten as $\gS_\vp = \{\vp \in \mathbb{R}^{h\times w\times c} \vert \|\vp\|_\infty \leq \epsilon_\infty,~0 \leq \vx + \vp \leq 1\}$. Combining with structured sparsity constraint, we can generate alleged \textit{adversarial watermarks} with our method. 
\begin{table}[h]
\vspace{-1em}
    \centering
    \caption{\small Robust accuracy of vanilla model (RN-34) against adversarial watermarks with different $l_\infty$ adversarial budgets. The experiment is conducted on ImageNet-100.}
    \label{tab:wm}
    \vspace{-1em}
    \subtable[18$\times$18 circle]{
    \small
    \begin{tabular}{c|c c c c}
        \toprule[1.5pt]
        $\epsilon_{\infty}$ & $8/255$ & $16/255$ & $32/255$ & $64/255$\\
        \midrule[1pt]
        RS & 59.0 & 44.6 & 32.4 & 26.4\\
        \textbf{sPGD}$_{\mathrm{p}}$ & 47.2 & 25.4 & 16.6 & 9.0\\
        \textbf{sPGD}$_{\mathrm{u}}$ & 56.4 & 39.6 & 25.8 & 16.8\\
        \textbf{sAA} & \textbf{44.2} & \textbf{24.0} &  \textbf{15.6} & \textbf{8.8}\\
        \bottomrule[1.5pt]
    \end{tabular}
    }
    \subtable[60$\times$60 letter ``A"]{
    \small
    \begin{tabular}{c|c c c c}
        \toprule[1.5pt]
        $\epsilon_{\infty}$ & $8/255$ & $16/255$ & $32/255$ & $64/255$\\
        \midrule[1pt]
        RS &14.2  & 12.8 & 9.2 & 5.8\\
        \textbf{sPGD}$_{\mathrm{p}}$ & 3.8 & 2.6 & 2.4 & 2.2\\
        \textbf{sPGD}$_{\mathrm{u}}$ & 7.4 & 4.8 & 2.8  & 2.0\\
        \textbf{sAA} &\textbf{2.0}  & \textbf{0.8} & \textbf{0.8}  & \textbf{0.6}\\
        \bottomrule[1.5pt]
    \end{tabular}
    }
    \vspace{-1em}
\end{table}

To evaluate the performance of our method in generating adversarial watermarks, we conduct experiments with different patterns and different $\l_\infty$ adversarial budgets. From Table \ref{tab:wm}, we can find that when the size of the pattern is large, like $60\times 60$ letter ``A" with $1747$ non-zero elements, our attack can still achieve decent performance under different $l_\infty$ adversarial budgets. In contrast, when we adopt a $18\times 18$ circle pattern with $208$ non-zero elements, the attack success rate abruptly declines as the $l_\infty$ adversarial budget decreases. Nevertheless, from the aspect of attack success rate, our approaches still outperform Sparse-RS by a large margin.

\subsection{Comparison under Different Iteration Numbers and Different Sparsity Levels} \label{sec:iter_k}
\begin{figure*}[h]
\vspace{-1em}
    \centering
    \subfigure[CIFAR-10, $\epsilon=20$]{\includegraphics[width=0.315\textwidth]{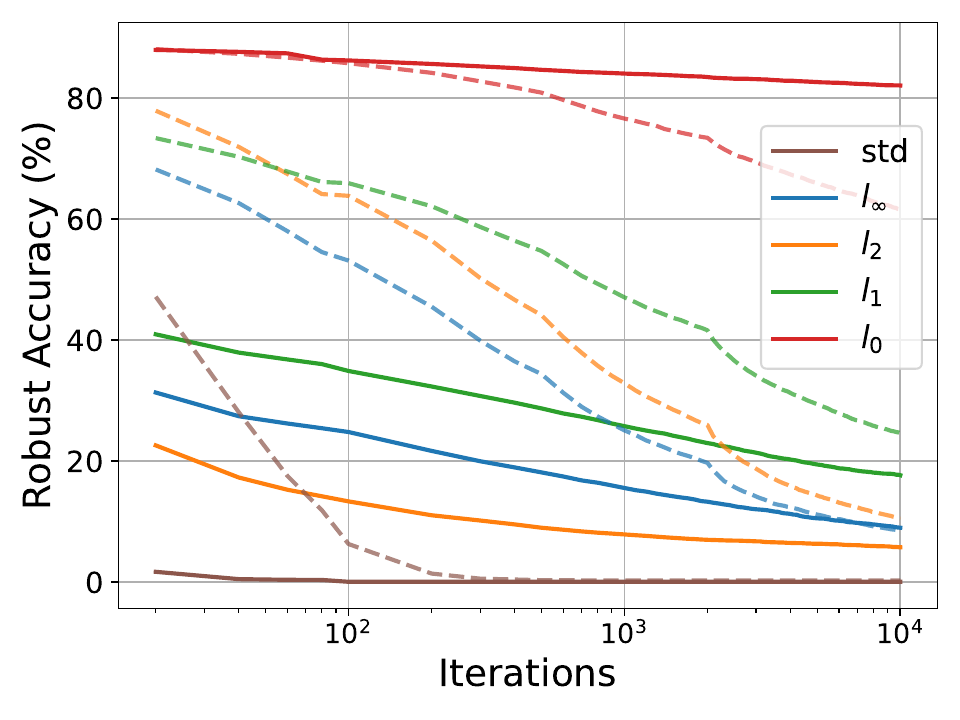}}
    \subfigure[ImageNet-100, $\epsilon=200$ ]{\includegraphics[width=0.315\textwidth]{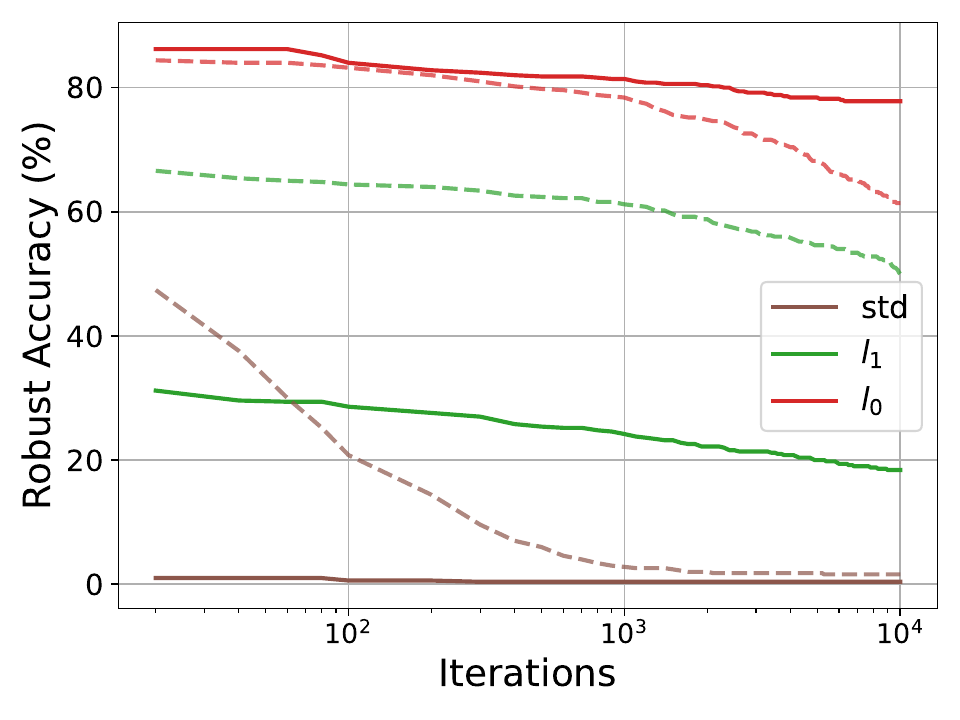}}
    \subfigure[CIFAR-10, $t=10000$]{\includegraphics[width=0.315\textwidth]{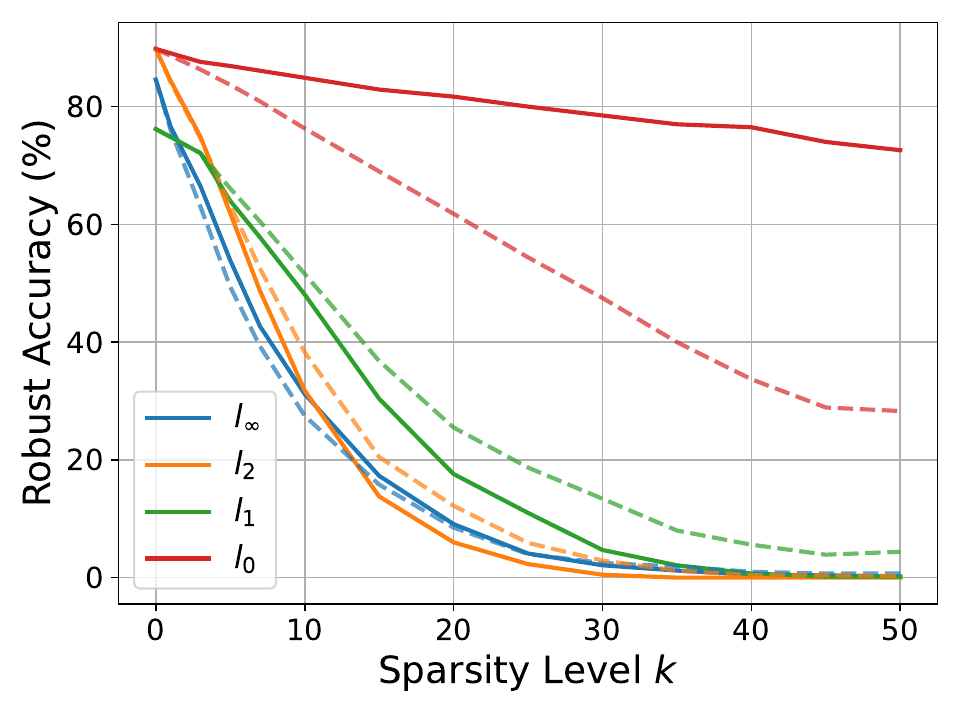}}
    \caption{\small Comparison between sPGD and RS attack under different iteration numbers and different sparsity levels. \textbf{(a)} Different iteration number comparison on CIFAR-10, $\epsilon=20$. ResNet18 (std), PORT ($l_\infty$ and $l_2$) \citep{sehwag2021}, $l_1$-APGD ($l_1$) \citep{croce2021mind} and sTRADES ($l_0$) are evaluated. \textbf{(b)} Different iteration number comparison on ImageNet-100, $\epsilon=200$. ResNet34 (std), Fast-EG-$l_1$ ($l_1$) \cite{jiang2023towards} and sAT ($l_0$) are evaluated. In (a) and (b), the total iteration number ranges from $20$ and $10000$. For better visualization, the x-axis is in the log scale. \textbf{(c)} Different sparsity comparison on CIFAR-10. The evaluated models are the same as those in (a). The $\epsilon$ ranges from $0$ and $50$. The number of total iterations is set to $10000$. \textbf{Note that the results of sPGD and RS attack are shown in solid lines and dotted lines, respectively.} }
    \label{fig:iter_k}
    \vspace{-1em}
\end{figure*}
In this subsection, we further compare our method sPGD, which is a white-box attack, with RS attack, the strongest black-box attack in the previous section. Specifically, we compare them under various iteration numbers on CIFAR-10 and ImageNet-100, which have different resolutions. In addition, we also compare sPGD and RS under different sparsity levels on CIFAR-10.

As illustrated in Figure \ref{fig:iter_k} (a) and (b), sTRADES has better performance than other robust models by a large margin in all iterations of both sPGD and RS attacks, which is consistent with the results in Table \ref{tab:20},~\ref{tab:more} and~VIII.
For vanilla and other robust models, although the performances of both sPGD and RS attack get improved with more iterations, sPGD significantly outperforms RS attack when the iteration number is small (e.g. $<1000$ iterations), which makes it feasible for adversarial training. Similar to other black-box attacks, the performance of RS attack drastically deteriorates when the query budget is limited. 
In addition, our proposed gradient-based sPGD significantly outperforms RS on ImageNet-100, where the search space is much larger than that on CIFAR-10, i.e., higher image resolution and higher sparsity level $\epsilon$. This suggests that our approach is scalable and shows higher efficiency on high-resolution images. 
Furthermore, although the performance of RS does not converge even when the iteration number reaches $10000$, a larger query budget will make it computationally impractical. Following the setting in \citet{croce2022sparse}, we do not consider a larger query budget in our experiments, either.

Furthermore, we can observe from Figure \ref{fig:iter_k} (c) that RS attack shows slightly better performance only on the $l_\infty$ model and when $\epsilon$ is small. The search space for the perturbed features is relatively small when $\epsilon$ is small, which facilitates heuristic black-box search methods like RS attack. As $\epsilon$ increases, sPGD outperforms RS attack in all cases until both attacks achieve almost $100\%$ attack success rate. 

\subsection{Efficiency of sPGD} \label{sec:app_runtime}
To further showcase the efficiency of our approach, we first compare the distributions of the number of iterations needed by sPGD and RS to successfully generate adversarial samples. Figure \ref{fig:iter_hist} illustrates the distribution of the iteration numbers needed by sPGD and RS to successfully generate adversarial samples. For $l_\infty$ and $l_1$ robust models, our proposed sPGD consumes distinctly fewer iteration numbers to successfully generate an adversarial sample than RS, the strongest black-box attack in Table \ref{tab:20}, while maintaining a high attack success rate. Similar to the observations in Figure \ref{fig:iter_k}, the model trained by sTRADES suffers from gradient masking. However, RS still requires a large query budget to successfully generate an adversarial sample. This further demonstrates the efficiency of sPGD, which makes adversarial training feasible.
\begin{figure*}[!ht]
    \centering
    \subfigure[$l_\infty$ PORT \cite{sehwag2021}]{\includegraphics[width=0.315\textwidth]{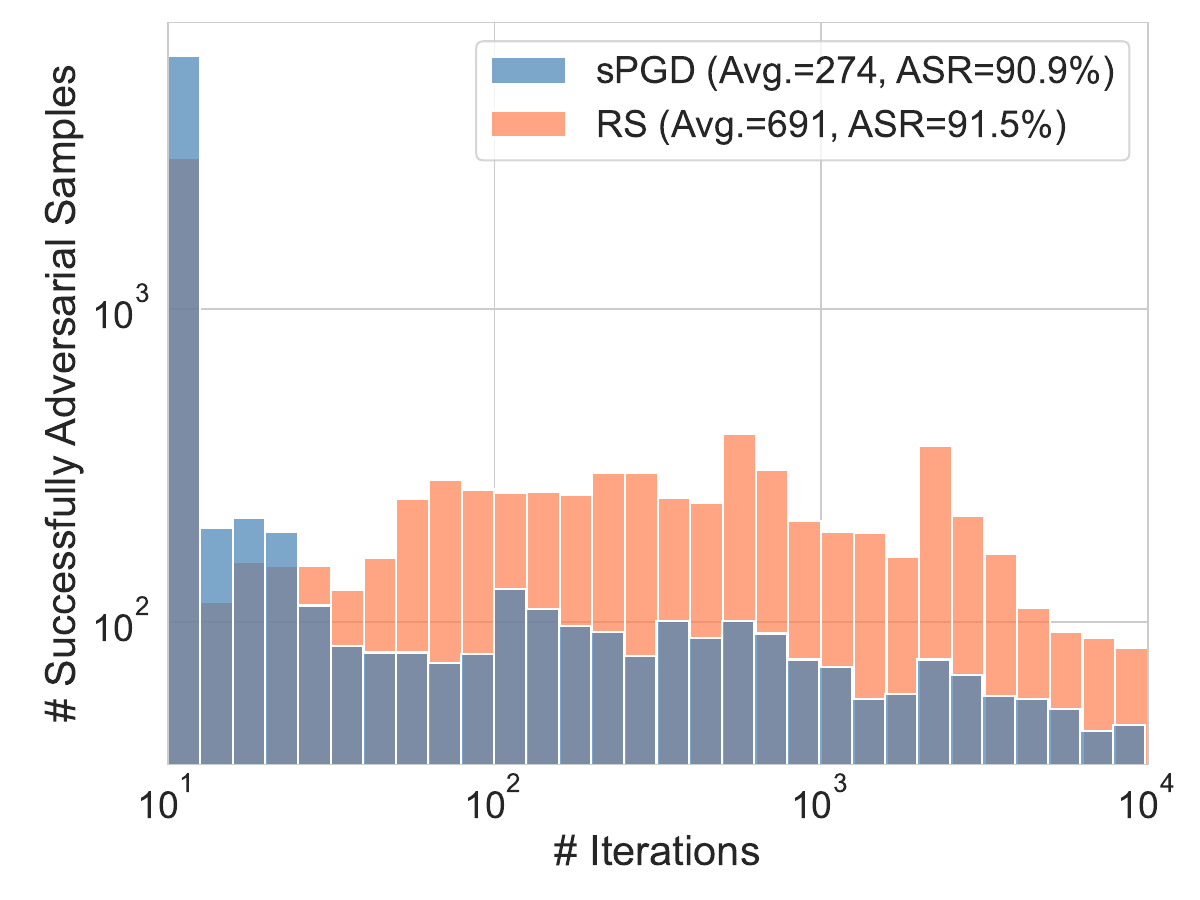}}
    \subfigure[$l_1$-APGD \cite{croce2021mind}]{\includegraphics[width=0.315\textwidth]{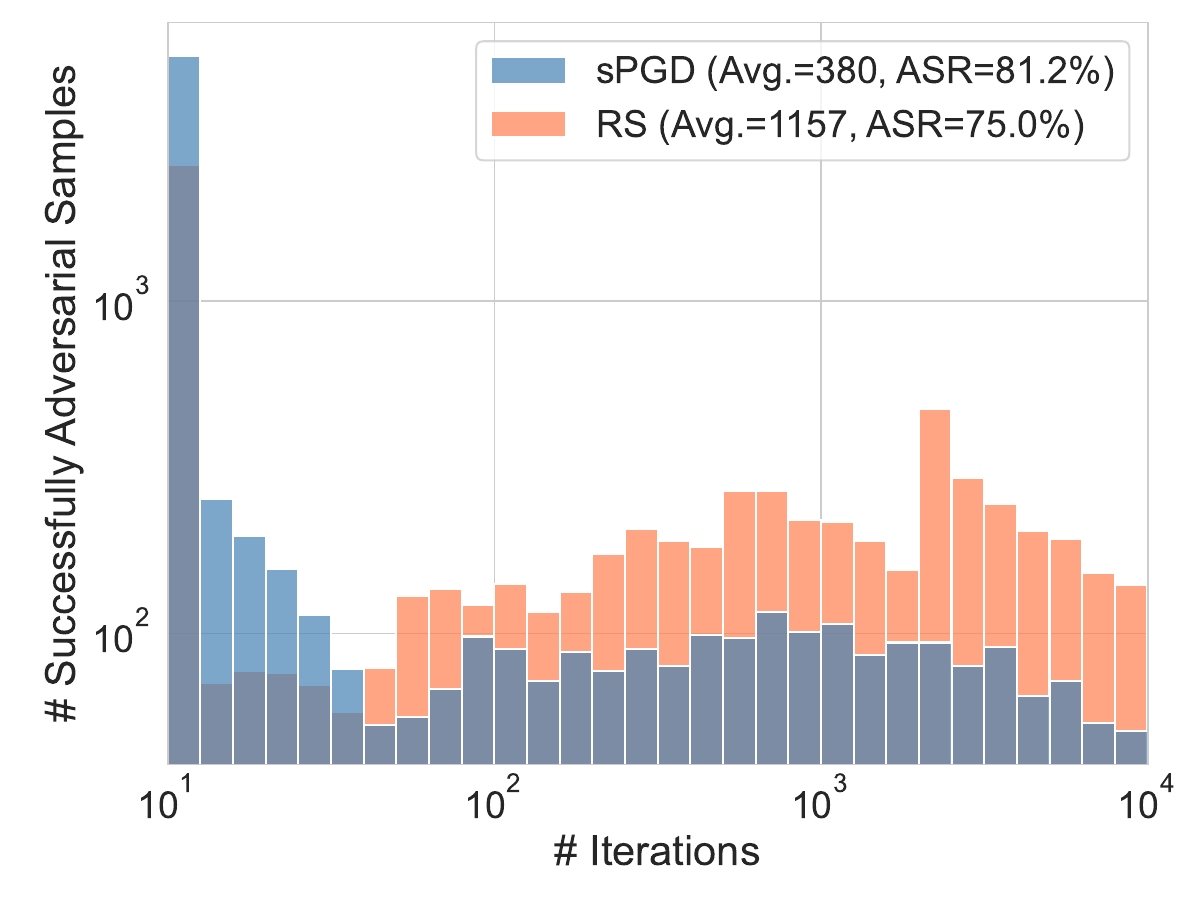}}
    \subfigure[sTRADES]{\includegraphics[width=0.315\textwidth]{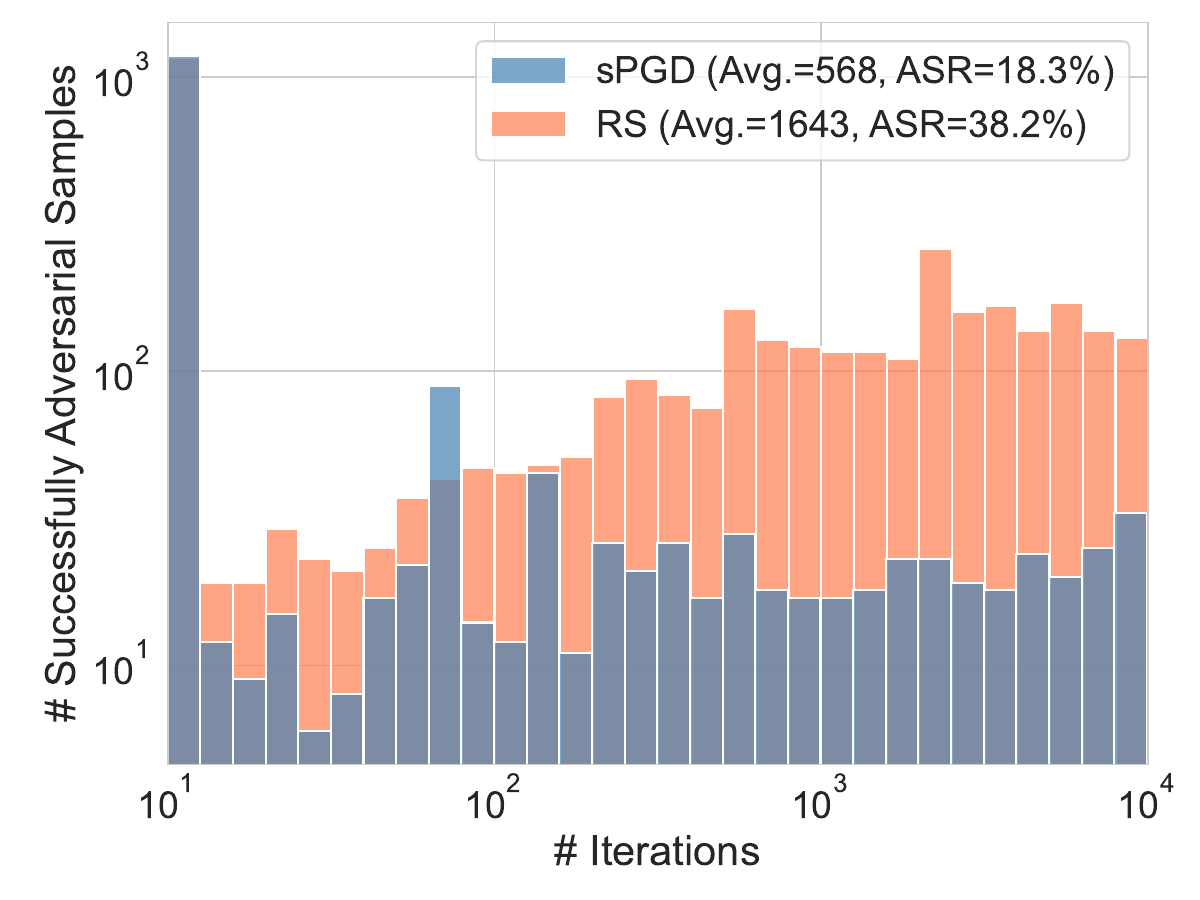}}
    \caption{\small Distribution of the iteration numbers needed by sPGD (blue) and RS (orange) to successfully generate adversarial samples. The results are obtained from different models: \textbf{(a)} $l_\infty$ PORT \cite{sehwag2021}, \textbf{(b)} $l_1$-APGD \cite{croce2021mind} and \textbf{(c)} our sPGD. The average iteration numbers (Avg.) and attack success rate (ASR), i.e., $1-$Robust Acc., are reported in the legend. For better visualization, we clip the minimum iteration number to $10$ and show the x- and y-axis in log scale.}
    \label{fig:iter_hist}
    \vspace{-1em}
\end{figure*}

Additionally, we compare the runtime of sPGD with other attacks in the same configuration as in Table \ref{tab:20}. As shown in Table \ref{tab:runtime}, the proposed sPGD shows the highest efficiency among various attacks. Although sAA consumes more time (approximately $2\times$ sPGD + RS), it can provide a reliable evaluation against $l_0$ bounded perturbation.
\begin{table}[h]
    \centering
    \small
    \caption{\small Runtime of different attacks on 1000 test instances with batch size 500. The sparsity level $\epsilon=20$. The evaluated model is sTRADES. The model is trained on \textbf{CIFAR-10}. The experiments are implemented on NVIDIA Tesla V100.}
    \small
    \label{tab:runtime}
    \begin{tabular}{c|c c c c }
        \toprule[1.5pt]
         Attack & CS & RS  & SF & PGD$_0$ \\
        \midrule[1pt]
         Runtime & 604 min & 59 min& 92 min &750 min\\
         \midrule[1.5pt]
          Attack & SAIF & \textbf{sPGD$_{\mathrm{p}}$}& \textbf{sPGD$_{\mathrm{u}}$}& \textbf{sAA}\\
        \midrule[1pt]
         Runtime & 122 min&\textbf{42 min}&45 min& 148 min\\
        \bottomrule[1.5pt]
    \end{tabular}
\end{table}

\subsection{Transferability of Adversarial Perturbations} \label{sec:trans}
\begin{table}[h]
\vspace{-1em}
    \centering
    \small
    \caption{\small Transferability of RS and sPGD between VGG11 (V) and ResNet18 (R) on CIFAR-10 with $\epsilon=$ 20. Attack success rate (ASR) is reported. The perturbations are generated on the source model (left), and are evaluated on the target model (right). Note that vanilla models are evaluated, and $\alpha$ and $\beta$ of sPGD are set to 0.75.}
    \small
    \label{tab:trans}
    \begin{tabular}{c|c c c c}
        \toprule[1.5pt]
        ~ & V$\rightarrow$V & V$\rightarrow$R & R$\rightarrow$R & R$\rightarrow$V \\
        \midrule[1pt]
        RS & 53.9 & 28.4 & 33.0 & 37.4\\
        sPGD$_{\mathrm{p}}$ & 58.0 & \textbf{43.9} & 50.8 & 48.0\\
        sPGD$_{\mathrm{u}}$ & \textbf{64.9} & 40.0 & \textbf{52.7} &\textbf{56.7} \\
        \bottomrule[1.5pt]
    \end{tabular}
    \vspace{-0.5em}
\end{table}
To evaluate the transferability of our attack across different models and architectures, we generate adversarial perturbations based on one model and report the attack success rate (ASR) on another model. As shown in Table \ref{tab:trans}, sPGD exhibits better transferability than the most competitive baseline, Sparse-RS (RS) \citep{croce2022sparse}. This further demonstrates the effectiveness of our method and its potential application in more practical scenarios.

\subsection{Attack Object Detection and Segmentation Models}

The versatility of our method allows it to be applied to a variety of tasks beyond classification, e.g., object detection and segmentation. For implementation, one could simply replace the loss function $\mathcal{L}$ in Algorithms \ref{alg:main} and \ref{alg:structure} with a specific detection or segmentation loss function.

In Table \ref{tab:yolo}, we report the mAP50-95 of YOLOv8-M \cite{varghese2024yolov8} in object detection and segmentation tasks when the input image is manipulated by sparse perturbations, including both structured and unstructured ones. 
In addition, we include the results of RS and SAIF as a comparison. Since SAIF is not able to generate structured sparse perturbations, we do not include its results in the experiments of generating patch perturbations.
The results indicate that sPGD, as an attack framework, is still effective in these two tasks. 
Considering higher resolution images in the COCO dataset \cite{lin2014microsoft} ($640\times 640$), we use approximately the same ratio between the number of perturbed pixels and the total number of pixels as in Table \ref{tab:more} when setting the perturbation budgets $\epsilon$, i.e., $0.4\%$ / $0.8\%$ of the pixels. 
Compared with structured perturbations, unstructured perturbations are spatially more dispersed, making them more effective for multi-object images in these tasks.


\begin{table}[h]
\vspace{-1em}
    \centering
    \small
    \caption{\small Evaluation of YOLOv8-M performance on clean and adversarial COCO dataset \cite{lin2014microsoft} in object detection and segmentation tasks. The metric is mAP50-95. We test the performance on the adversarial examples generated by different attacks, where both unstructured and structured cases are considered.}\label{tab:yolo}
    \vspace{-1em}
    \subtable[\small \textbf{Object Detection}]{
    \small
    \begin{tabular}{c|c|c c c c}
        \toprule[1.5pt]
        Mode & $\epsilon$& Clean & RS & SAIF & sPGD\\
        \midrule[1pt]
        \multirow{2}{*}{$l_0$}& 1600 & 55.1 & 17.1 & 15.8 & 15.5 \\
        ~ & 3200 & 55.1 & 8.1 & 7.4&6.6\\
        \midrule[1pt]
        \multirow{2}{*}{$4\times 4$ patch}& 100 & 55.1 & 47.1 & - & 40.1\\
        ~ & 200 &55.1 & 40.8& - & 34.5\\
        \bottomrule[1.5pt]
    \end{tabular}
    }
    \subtable[\small \textbf{Segmentation}]{
    \small
    \begin{tabular}{c|c|c c c c}
        \toprule[1.5pt]
        Mode & $\epsilon$& Clean & RS & SAIF & sPGD\\
        \midrule[1pt]
        \multirow{2}{*}{$l_0$}& 1600 & 53.6 & 17.6 & 16.4 & 15.6 \\
        ~ & 3200 & 53.6 & 8.8 & 8.2 &7.4\\
        \midrule[1pt]
        \multirow{2}{*}{$4\times 4$ patch}& 100 & 53.6 & 46.5 & - & 39.1\\
        ~ & 200 &53.6 & 39.7& - & 33.1\\
        \bottomrule[1.5pt]
    \end{tabular}
    }
    \vspace{-1em}
\end{table}

\subsection{Visualization} \label{sec:vis}
\begin{figure*}[ht]
    \centering
    \subfigure[$l_0$, $\epsilon=200$]{\includegraphics[width=0.23\textwidth]{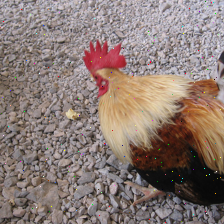}}
    \subfigure[$l_0$, $\epsilon=200$]{\includegraphics[width=0.23\textwidth]{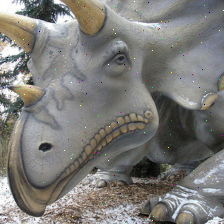}}
    \subfigure[row, $\epsilon=1$]{\includegraphics[width=0.23\textwidth]{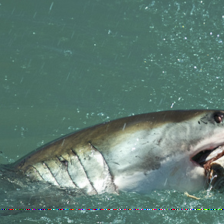}}
    \subfigure[row, $\epsilon=1$]{\includegraphics[width=0.23\textwidth]{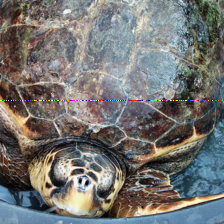}}
    \\
    \subfigure[star, $25\times25$, $\epsilon=1$]{\includegraphics[width=0.23\textwidth]{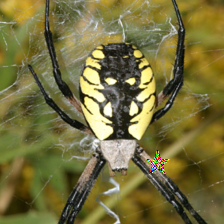}}
    \subfigure[star, $25\times25$, $\epsilon=1$]{\includegraphics[width=0.23\textwidth]{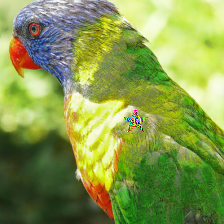}}
    \subfigure[``A", $60\times60$, $\epsilon_\infty=32/255$]{\includegraphics[width=0.23\textwidth]{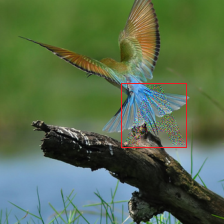}}
    \subfigure[``A", $60\times60$, $\epsilon_\infty=32/255$]{\includegraphics[width=0.23\textwidth]{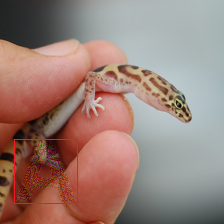}}
    \caption{\small Visualization of different sparse adversarial examples in ImageNet-100. The model is vanilla ResNet-34. The predictions before (left) and after (right) attack are (a) cock$\rightarrow$hen, (b) triceratops$\rightarrow$spotted salamander, (c) great white shark$\rightarrow$tench, (d) loggerhead$\rightarrow$leatherback turtle, (e) black and gold garden spider$\rightarrow$garden spider, (f) lorikeet$\rightarrow$macaw, (g) bee eater$\rightarrow$coucal, and (h) banded gecko$\rightarrow$alligator lizard. \textbf{Note that the red squares in (g) and (h) are just for highlighting the perturbation position and not part of perturbations.}}
    \vspace{-1em}
    \label{fig:visual}
\end{figure*}

For a better visualization, we present some adversarial examples with different types of sparse perturbations in Figure \ref{fig:visual}, including $l_0$ bounded perturbations, row-wise perturbations, star-like perturbations, and adversarial watermarks of letter ``A". The attack is sPGD, and the model is vanilla ResNet-34 trained on ImageNet-100. More adversarial examples are shown in Appendix D.
\section{Conclusion}
In this paper, we propose an effective and efficient white-box attack named sPGD to generate both unstructured and structured perturbations. sPGD obtains the state-of-the-art performance among white-box attacks. Based on this, we combine it with black-box attacks for more comprehensive and reliable evaluation of robustness against sparse perturbations. Our proposed sPGD is particularly effective in the realm of limited iteration numbers. Due to its efficiency, we incorporate sPGD into the framework of adversarial training to obtain robust models against sparse perturbations. The models trained with our method demonstrate the best robust accuracy. 

\section*{Acknowledgments}

This work is supported by the NSFC project (No. 62306250). It is supported by the internal funds of City University of Hong Kong (No. 9610614 and No. 9229130). We also thank Yixiao Huang for his contribution to this work.
\bibliographystyle{IEEEtranN}
\bibliography{references}

\begin{thebibliography}{73}
\providecommand{\natexlab}[1]{#1}
\providecommand{\url}[1]{#1}
\csname url@samestyle\endcsname
\providecommand{\newblock}{\relax}
\providecommand{\bibinfo}[2]{#2}
\providecommand{\BIBentrySTDinterwordspacing}{\spaceskip=0pt\relax}
\providecommand{\BIBentryALTinterwordstretchfactor}{4}
\providecommand{\BIBentryALTinterwordspacing}{\spaceskip=\fontdimen2\font plus
\BIBentryALTinterwordstretchfactor\fontdimen3\font minus \fontdimen4\font\relax}
\providecommand{\BIBforeignlanguage}[2]{{%
\expandafter\ifx\csname l@#1\endcsname\relax
\typeout{** WARNING: IEEEtranN.bst: No hyphenation pattern has been}%
\typeout{** loaded for the language `#1'. Using the pattern for}%
\typeout{** the default language instead.}%
\else
\language=\csname l@#1\endcsname
\fi
#2}}
\providecommand{\BIBdecl}{\relax}
\BIBdecl

\bibitem[Szegedy et~al.(2013)Szegedy, Zaremba, Sutskever, Bruna, Erhan, Goodfellow, and Fergus]{2013Intriguing}
C.~Szegedy, W.~Zaremba, I.~Sutskever, J.~Bruna, D.~Erhan, I.~Goodfellow, and R.~Fergus, ``Intriguing properties of neural networks,'' \emph{Computer Science}, 2013.

\bibitem[Kurakin et~al.(2016)Kurakin, Goodfellow, and Bengio]{2016Adversarial}
A.~Kurakin, I.~Goodfellow, and S.~Bengio, ``Adversarial examples in the physical world,'' 2016.

\bibitem[Goodfellow et~al.(2014)Goodfellow, Shlens, and Szegedy]{2014Explaining}
I.~J. Goodfellow, J.~Shlens, and C.~Szegedy, ``Explaining and harnessing adversarial examples,'' \emph{Computer Science}, 2014.

\bibitem[Madry et~al.(2017)Madry, Makelov, Schmidt, Tsipras, and Vladu]{madry2017towards}
A.~Madry, A.~Makelov, L.~Schmidt, D.~Tsipras, and A.~Vladu, ``Towards deep learning models resistant to adversarial attacks,'' \emph{arXiv preprint arXiv:1706.06083}, 2017.

\bibitem[Zhang et~al.(2019{\natexlab{a}})Zhang, Yu, Jiao, Xing, and Jordan]{2019Theoretically}
H.~Zhang, Y.~Yu, J.~Jiao, E.~P. Xing, and M.~I. Jordan, ``Theoretically principled trade-off between robustness and accuracy,'' 2019.

\bibitem[Croce et~al.(2020)Croce, Andriushchenko, Sehwag, Debenedetti, Flammarion, Chiang, Mittal, and Hein]{croce2020robustbench}
F.~Croce, M.~Andriushchenko, V.~Sehwag, E.~Debenedetti, N.~Flammarion, M.~Chiang, P.~Mittal, and M.~Hein, ``Robustbench: a standardized adversarial robustness benchmark,'' \emph{arXiv preprint arXiv:2010.09670}, 2020.

\bibitem[Papernot et~al.(2017)Papernot, McDaniel, Goodfellow, Jha, Celik, and Swami]{papernot2017practical}
N.~Papernot, P.~McDaniel, I.~Goodfellow, S.~Jha, Z.~B. Celik, and A.~Swami, ``Practical black-box attacks against machine learning,'' in \emph{Proceedings of the 2017 ACM on Asia conference on computer and communications security}, 2017, pp. 506--519.

\bibitem[Akhtar and Mian(2018)]{Akhtar2018ThreatOA}
\BIBentryALTinterwordspacing
N.~Akhtar and A.~S. Mian, ``Threat of adversarial attacks on deep learning in computer vision: A survey,'' \emph{IEEE Access}, vol.~6, pp. 14\,410--14\,430, 2018. [Online]. Available: \url{https://api.semanticscholar.org/CorpusID:3536399}
\BIBentrySTDinterwordspacing

\bibitem[Xu et~al.(2019)Xu, Ma, Liu, Deb, Liu, Tang, and Jain]{Xu2019AdversarialAA}
\BIBentryALTinterwordspacing
H.~Xu, Y.~Ma, H.~Liu, D.~Deb, H.~Liu, J.~Tang, and A.~K. Jain, ``Adversarial attacks and defenses in images, graphs and text: A review,'' \emph{International Journal of Automation and Computing}, vol.~17, pp. 151 -- 178, 2019. [Online]. Available: \url{https://api.semanticscholar.org/CorpusID:202660800}
\BIBentrySTDinterwordspacing

\bibitem[Feng et~al.(2022)Feng, Mangaokar, Chen, Fernandes, Jha, and Prakash]{feng2022graphite}
R.~Feng, N.~Mangaokar, J.~Chen, E.~Fernandes, S.~Jha, and A.~Prakash, ``Graphite: Generating automatic physical examples for machine-learning attacks on computer vision systems,'' in \emph{2022 IEEE 7th European symposium on security and privacy (EuroS\&P)}.\hskip 1em plus 0.5em minus 0.4em\relax IEEE, 2022, pp. 664--683.

\bibitem[Wei et~al.(2023)Wei, Huang, Sun, and Yu]{wei2023unified}
X.~Wei, Y.~Huang, Y.~Sun, and J.~Yu, ``Unified adversarial patch for cross-modal attacks in the physical world,'' in \emph{Proceedings of the IEEE/CVF International Conference on Computer Vision}, 2023, pp. 4445--4454.

\bibitem[Modas et~al.(2018)Modas, Moosavi-Dezfooli, and Frossard]{2018SparseFool}
A.~Modas, S.~M. Moosavi-Dezfooli, and P.~Frossard, ``Sparsefool: a few pixels make a big difference,'' 2018.

\bibitem[Croce and Hein(2019{\natexlab{a}})]{croce2019sparse}
F.~Croce and M.~Hein, ``Sparse and imperceivable adversarial attacks,'' in \emph{Proceedings of the IEEE/CVF international conference on computer vision}, 2019, pp. 4724--4732.

\bibitem[Su et~al.(2019)Su, Vargas, and Sakurai]{su2019one}
J.~Su, D.~V. Vargas, and K.~Sakurai, ``One pixel attack for fooling deep neural networks,'' \emph{IEEE Transactions on Evolutionary Computation}, vol.~23, no.~5, pp. 828--841, 2019.

\bibitem[Dong et~al.(2020)Dong, Chen, Bao, Qin, Yuan, Zhang, Yu, and Chen]{2020GreedyFool}
X.~Dong, D.~Chen, J.~Bao, C.~Qin, L.~Yuan, W.~Zhang, N.~Yu, and D.~Chen, ``Greedyfool: Distortion-aware sparse adversarial attack,'' 2020.

\bibitem[Rao et~al.(2020)Rao, Stutz, and Schiele]{rao2020adversarial}
S.~Rao, D.~Stutz, and B.~Schiele, ``Adversarial training against location-optimized adversarial patches,'' in \emph{European conference on computer vision}.\hskip 1em plus 0.5em minus 0.4em\relax Springer, 2020, pp. 429--448.

\bibitem[Croce et~al.(2022)Croce, Andriushchenko, Singh, Flammarion, and Hein]{croce2022sparse}
F.~Croce, M.~Andriushchenko, N.~D. Singh, N.~Flammarion, and M.~Hein, ``Sparse-rs: a versatile framework for query-efficient sparse black-box adversarial attacks,'' in \emph{Proceedings of the AAAI Conference on Artificial Intelligence}, vol.~36, no.~6, 2022, pp. 6437--6445.

\bibitem[Tramer and Boneh(2019)]{tramer2019adversarial}
F.~Tramer and D.~Boneh, ``Adversarial training and robustness for multiple perturbations,'' \emph{Advances in neural information processing systems}, vol.~32, 2019.

\bibitem[Croce and Hein(2021)]{croce2021mind}
F.~Croce and M.~Hein, ``Mind the box: $ l\_1 $-apgd for sparse adversarial attacks on image classifiers,'' in \emph{International Conference on Machine Learning}.\hskip 1em plus 0.5em minus 0.4em\relax PMLR, 2021, pp. 2201--2211.

\bibitem[Jiang et~al.(2023)Jiang, Liu, Huang, Salzmann, and Süsstrunk]{jiang2023towards}
Y.~Jiang, C.~Liu, Z.~Huang, M.~Salzmann, and S.~Süsstrunk, ``Towards stable and efficient adversarial training against $l_1$ bounded adversarial attacks,'' in \emph{International Conference on Machine Learning}.\hskip 1em plus 0.5em minus 0.4em\relax PMLR, 2023.

\bibitem[Brown et~al.(2017)Brown, Man{\'{e}}, Roy, Abadi, and Gilmer]{brown2017adversarial}
T.~B. Brown, D.~Man{\'{e}}, A.~Roy, M.~Abadi, and J.~Gilmer, ``Adversarial patch,'' \emph{NeurIPS 2017 Workshop on Machine Learning and Computer Security}, 2017.

\bibitem[Karmon et~al.(2018)Karmon, Zoran, and Goldberg]{karmon2018lavan}
\BIBentryALTinterwordspacing
D.~Karmon, D.~Zoran, and Y.~Goldberg, ``{L}a{VAN}: Localized and visible adversarial noise,'' in \emph{Proceedings of the 35th International Conference on Machine Learning}, ser. Proceedings of Machine Learning Research, J.~Dy and A.~Krause, Eds., vol.~80.\hskip 1em plus 0.5em minus 0.4em\relax PMLR, 10--15 Jul 2018, pp. 2507--2515. [Online]. Available: \url{https://proceedings.mlr.press/v80/karmon18a.html}
\BIBentrySTDinterwordspacing

\bibitem[Yang et~al.(2020)Yang, Kortylewski, Xie, Cao, and Yuille]{yang2020patch}
C.~Yang, A.~Kortylewski, C.~Xie, Y.~Cao, and A.~Yuille, ``Patchattack: A black-box texture-based attack with reinforcement learning,'' in \emph{Computer Vision -- ECCV 2020}, A.~Vedaldi, H.~Bischof, T.~Brox, and J.-M. Frahm, Eds.\hskip 1em plus 0.5em minus 0.4em\relax Cham: Springer International Publishing, 2020, pp. 681--698.

\bibitem[Bach(2010)]{bach2010structured}
F.~Bach, ``Structured sparsity-inducing norms through submodular functions,'' \emph{Advances in Neural Information Processing Systems}, vol.~23, 2010.

\bibitem[Zhang et~al.(2019{\natexlab{b}})Zhang, Yu, Jiao, Xing, Ghaoui, and Jordan]{Zhang2019TheoreticallyPT}
\BIBentryALTinterwordspacing
H.~Zhang, Y.~Yu, J.~Jiao, E.~P. Xing, L.~E. Ghaoui, and M.~I. Jordan, ``Theoretically principled trade-off between robustness and accuracy,'' \emph{ArXiv}, vol. abs/1901.08573, 2019. [Online]. Available: \url{https://api.semanticscholar.org/CorpusID:59222747}
\BIBentrySTDinterwordspacing

\bibitem[Weng et~al.(2018)Weng, Zhang, Chen, Song, Hsieh, Boning, Dhillon, and Daniel]{Weng2018TowardsFC}
\BIBentryALTinterwordspacing
T.-W. Weng, H.~Zhang, H.~Chen, Z.~Song, C.-J. Hsieh, D.~S. Boning, I.~S. Dhillon, and L.~Daniel, ``Towards fast computation of certified robustness for relu networks,'' \emph{ArXiv}, vol. abs/1804.09699, 2018. [Online]. Available: \url{https://api.semanticscholar.org/CorpusID:13750928}
\BIBentrySTDinterwordspacing

\bibitem[Kurakin et~al.(2017)Kurakin, Goodfellow, and Bengio]{kurakin2017adversarial}
\BIBentryALTinterwordspacing
A.~Kurakin, I.~J. Goodfellow, and S.~Bengio, ``Adversarial machine learning at scale,'' in \emph{International Conference on Learning Representations}, 2017. [Online]. Available: \url{https://openreview.net/forum?id=BJm4T4Kgx}
\BIBentrySTDinterwordspacing

\bibitem[Bhagoji et~al.(2018)Bhagoji, He, Li, and Song]{Bhagoji2018PracticalBA}
\BIBentryALTinterwordspacing
A.~N. Bhagoji, W.~He, B.~Li, and D.~X. Song, ``Practical black-box attacks on deep neural networks using efficient query mechanisms,'' in \emph{European Conference on Computer Vision}, 2018. [Online]. Available: \url{https://api.semanticscholar.org/CorpusID:52951839}
\BIBentrySTDinterwordspacing

\bibitem[Ilyas et~al.(2018{\natexlab{a}})Ilyas, Engstrom, Athalye, and Lin]{Ilyas2018BlackboxAA}
\BIBentryALTinterwordspacing
A.~Ilyas, L.~Engstrom, A.~Athalye, and J.~Lin, ``Black-box adversarial attacks with limited queries and information,'' in \emph{International Conference on Machine Learning}, 2018. [Online]. Available: \url{https://api.semanticscholar.org/CorpusID:5046541}
\BIBentrySTDinterwordspacing

\bibitem[Ilyas et~al.(2018{\natexlab{b}})Ilyas, Engstrom, and Madry]{Ilyas2018PriorCB}
\BIBentryALTinterwordspacing
A.~Ilyas, L.~Engstrom, and A.~Madry, ``Prior convictions: Black-box adversarial attacks with bandits and priors,'' \emph{ArXiv}, vol. abs/1807.07978, 2018. [Online]. Available: \url{https://api.semanticscholar.org/CorpusID:49907212}
\BIBentrySTDinterwordspacing

\bibitem[Tu et~al.(2018)Tu, Ting, Chen, Liu, Zhang, Yi, Hsieh, and Cheng]{Tu2018AutoZOOMAZ}
\BIBentryALTinterwordspacing
C.-C. Tu, P.-S. Ting, P.-Y. Chen, S.~Liu, H.~Zhang, J.~Yi, C.-J. Hsieh, and S.-M. Cheng, ``Autozoom: Autoencoder-based zeroth order optimization method for attacking black-box neural networks,'' in \emph{AAAI Conference on Artificial Intelligence}, 2018. [Online]. Available: \url{https://api.semanticscholar.org/CorpusID:44079102}
\BIBentrySTDinterwordspacing

\bibitem[Uesato et~al.(2018)Uesato, O'Donoghue, van~den Oord, and Kohli]{Uesato2018AdversarialRA}
\BIBentryALTinterwordspacing
J.~Uesato, B.~O'Donoghue, A.~van~den Oord, and P.~Kohli, ``Adversarial risk and the dangers of evaluating against weak attacks,'' \emph{ArXiv}, vol. abs/1802.05666, 2018. [Online]. Available: \url{https://api.semanticscholar.org/CorpusID:3639844}
\BIBentrySTDinterwordspacing

\bibitem[Alzantot et~al.(2018)Alzantot, Sharma, Chakraborty, and Srivastava]{Alzantot2018GenAttackPB}
\BIBentryALTinterwordspacing
M.~F. Alzantot, Y.~Sharma, S.~Chakraborty, and M.~B. Srivastava, ``Genattack: practical black-box attacks with gradient-free optimization,'' \emph{Proceedings of the Genetic and Evolutionary Computation Conference}, 2018. [Online]. Available: \url{https://api.semanticscholar.org/CorpusID:44166696}
\BIBentrySTDinterwordspacing

\bibitem[Guo et~al.(2019)Guo, Gardner, You, Wilson, and Weinberger]{Guo2019SimpleBA}
\BIBentryALTinterwordspacing
C.~Guo, J.~R. Gardner, Y.~You, A.~G. Wilson, and K.~Q. Weinberger, ``Simple black-box adversarial attacks,'' \emph{ArXiv}, vol. abs/1905.07121, 2019. [Online]. Available: \url{https://api.semanticscholar.org/CorpusID:86541092}
\BIBentrySTDinterwordspacing

\bibitem[Al-Dujaili and O’Reilly(2019)]{AlDujaili2019ThereAN}
\BIBentryALTinterwordspacing
A.~Al-Dujaili and U.-M. O’Reilly, ``There are no bit parts for sign bits in black-box attacks,'' \emph{ArXiv}, vol. abs/1902.06894, 2019. [Online]. Available: \url{https://api.semanticscholar.org/CorpusID:67749599}
\BIBentrySTDinterwordspacing

\bibitem[Moon et~al.(2019)Moon, An, and Song]{Moon2019ParsimoniousBA}
\BIBentryALTinterwordspacing
S.~Moon, G.~An, and H.~O. Song, ``Parsimonious black-box adversarial attacks via efficient combinatorial optimization,'' \emph{ArXiv}, vol. abs/1905.06635, 2019. [Online]. Available: \url{https://api.semanticscholar.org/CorpusID:155100229}
\BIBentrySTDinterwordspacing

\bibitem[Meunier et~al.(2019)Meunier, Atif, and Teytaud]{Meunier2019YetAB}
\BIBentryALTinterwordspacing
L.~Meunier, J.~Atif, and O.~Teytaud, ``Yet another but more efficient black-box adversarial attack: tiling and evolution strategies,'' \emph{ArXiv}, vol. abs/1910.02244, 2019. [Online]. Available: \url{https://api.semanticscholar.org/CorpusID:203837562}
\BIBentrySTDinterwordspacing

\bibitem[Andriushchenko et~al.(2019)Andriushchenko, Croce, Flammarion, and Hein]{Andriushchenko2019SquareAA}
\BIBentryALTinterwordspacing
M.~Andriushchenko, F.~Croce, N.~Flammarion, and M.~Hein, ``Square attack: a query-efficient black-box adversarial attack via random search,'' \emph{ArXiv}, vol. abs/1912.00049, 2019. [Online]. Available: \url{https://api.semanticscholar.org/CorpusID:208527215}
\BIBentrySTDinterwordspacing

\bibitem[Croce and Hein(2020)]{croce2020reliable}
F.~Croce and M.~Hein, ``Reliable evaluation of adversarial robustness with an ensemble of diverse parameter-free attacks,'' in \emph{International conference on machine learning}.\hskip 1em plus 0.5em minus 0.4em\relax PMLR, 2020, pp. 2206--2216.

\bibitem[Carlini and Wagner(2017)]{carlini2017towards}
N.~Carlini and D.~Wagner, ``Towards evaluating the robustness of neural networks,'' in \emph{2017 ieee symposium on security and privacy (sp)}.\hskip 1em plus 0.5em minus 0.4em\relax Ieee, 2017, pp. 39--57.

\bibitem[Imtiaz et~al.(2022)Imtiaz, Kohler, Miller, Wang, Sznaier, Camps, and Dy]{imtiaz2022saif}
T.~Imtiaz, M.~Kohler, J.~Miller, Z.~Wang, M.~Sznaier, O.~Camps, and J.~Dy, ``Saif: Sparse adversarial and interpretable attack framework,'' \emph{arXiv preprint arXiv:2212.07495}, 2022.

\bibitem[Frank et~al.(1956)Frank, Wolfe, et~al.]{frank1956algorithm}
M.~Frank, P.~Wolfe \emph{et~al.}, ``An algorithm for quadratic programming,'' \emph{Naval research logistics quarterly}, vol.~3, no. 1-2, pp. 95--110, 1956.

\bibitem[Croce and Hein(2019{\natexlab{b}})]{Croce2019MinimallyDA}
\BIBentryALTinterwordspacing
F.~Croce and M.~Hein, ``Minimally distorted adversarial examples with a fast adaptive boundary attack,'' \emph{ArXiv}, vol. abs/1907.02044, 2019. [Online]. Available: \url{https://api.semanticscholar.org/CorpusID:195791557}
\BIBentrySTDinterwordspacing

\bibitem[Sehwag et~al.(2021)Sehwag, Mahloujifar, Handina, Dai, Xiang, Chiang, and Mittal]{sehwag2021}
V.~Sehwag, S.~Mahloujifar, T.~Handina, S.~Dai, C.~Xiang, M.~Chiang, and P.~Mittal, ``Robust learning meets generative models: Can proxy distributions improve adversarial robustness?'' \emph{arXiv preprint arXiv:2104.09425}, 2021.

\bibitem[Rebuffi et~al.(2021)Rebuffi, Gowal, Calian, Stimberg, Wiles, and Mann]{rebuffi2021data}
S.~Rebuffi, S.~Gowal, D.~A. Calian, F.~Stimberg, O.~Wiles, and T.~A. Mann, ``Fixing data augmentation to improve adversarial robustness,'' \emph{CoRR}, vol. abs/2103.01946, 2021.

\bibitem[Gowal et~al.(2021)Gowal, Rebuffi, Wiles, Stimberg, Calian, and Mann]{2021Improving}
S.~Gowal, S.~Rebuffi, O.~Wiles, F.~Stimberg, D.~A. Calian, and T.~A. Mann, ``Improving robustness using generated data,'' \emph{CoRR}, vol. abs/2110.09468, 2021.

\bibitem[Rade and Moosavi-Dezfooli(2021)]{rade2021helperbased}
\BIBentryALTinterwordspacing
R.~Rade and S.-M. Moosavi-Dezfooli, ``Helper-based adversarial training: Reducing excessive margin to achieve a better accuracy vs. robustness trade-off,'' in \emph{ICML 2021 Workshop on Adversarial Machine Learning}, 2021. [Online]. Available: \url{https://openreview.net/forum?id=BuD2LmNaU3a}
\BIBentrySTDinterwordspacing

\bibitem[Cui et~al.(2023)Cui, Tian, Zhong, Qi, Yu, and Zhang]{Cui2023DecoupledKD}
\BIBentryALTinterwordspacing
J.~Cui, Z.~Tian, Z.~Zhong, X.~Qi, B.~Yu, and H.~Zhang, ``Decoupled kullback-leibler divergence loss,'' \emph{ArXiv}, vol. abs/2305.13948, 2023. [Online]. Available: \url{https://api.semanticscholar.org/CorpusID:258841423}
\BIBentrySTDinterwordspacing

\bibitem[Wang et~al.(2023)Wang, Pang, Du, Lin, Liu, and Yan]{Wang2023BetterDM}
\BIBentryALTinterwordspacing
Z.~Wang, T.~Pang, C.~Du, M.~Lin, W.~Liu, and S.~Yan, ``Better diffusion models further improve adversarial training,'' \emph{ArXiv}, vol. abs/2302.04638, 2023. [Online]. Available: \url{https://api.semanticscholar.org/CorpusID:256697167}
\BIBentrySTDinterwordspacing

\bibitem[Athalye et~al.(2018)Athalye, Carlini, and Wagner]{Athalye2018ObfuscatedGG}
\BIBentryALTinterwordspacing
A.~Athalye, N.~Carlini, and D.~A. Wagner, ``Obfuscated gradients give a false sense of security: Circumventing defenses to adversarial examples,'' in \emph{International Conference on Machine Learning}, 2018. [Online]. Available: \url{https://api.semanticscholar.org/CorpusID:3310672}
\BIBentrySTDinterwordspacing

\bibitem[Shafahi et~al.(2019)Shafahi, Najibi, Ghiasi, Xu, Dickerson, Studer, Davis, Taylor, and Goldstein]{Shafahi2019AdversarialTF}
\BIBentryALTinterwordspacing
A.~Shafahi, M.~Najibi, A.~Ghiasi, Z.~Xu, J.~P. Dickerson, C.~Studer, L.~S. Davis, G.~Taylor, and T.~Goldstein, ``Adversarial training for free!'' in \emph{Neural Information Processing Systems}, 2019. [Online]. Available: \url{https://api.semanticscholar.org/CorpusID:139102395}
\BIBentrySTDinterwordspacing

\bibitem[Zhang et~al.(2019{\natexlab{c}})Zhang, Zhang, Lu, Zhu, and Dong]{Zhang2019YouOP}
\BIBentryALTinterwordspacing
D.~Zhang, T.~Zhang, Y.~Lu, Z.~Zhu, and B.~Dong, ``You only propagate once: Accelerating adversarial training via maximal principle,'' in \emph{Neural Information Processing Systems}, 2019. [Online]. Available: \url{https://api.semanticscholar.org/CorpusID:146120969}
\BIBentrySTDinterwordspacing

\bibitem[Wong et~al.(2020)Wong, Rice, and Kolter]{Wong2020FastIB}
\BIBentryALTinterwordspacing
E.~Wong, L.~Rice, and J.~Z. Kolter, ``Fast is better than free: Revisiting adversarial training,'' \emph{ArXiv}, vol. abs/2001.03994, 2020. [Online]. Available: \url{https://api.semanticscholar.org/CorpusID:210164926}
\BIBentrySTDinterwordspacing

\bibitem[Sriramanan et~al.(2021)Sriramanan, Addepalli, Baburaj, and Babu]{Sriramanan2021TowardsEA}
\BIBentryALTinterwordspacing
G.~Sriramanan, S.~Addepalli, A.~Baburaj, and R.~V. Babu, ``Towards efficient and effective adversarial training,'' in \emph{Neural Information Processing Systems}, 2021. [Online]. Available: \url{https://api.semanticscholar.org/CorpusID:245261076}
\BIBentrySTDinterwordspacing

\bibitem[Kang and Moosavi-Dezfooli(2021)]{Kang2021UnderstandingCO}
\BIBentryALTinterwordspacing
P.~Kang and S.-M. Moosavi-Dezfooli, ``Understanding catastrophic overfitting in adversarial training,'' \emph{ArXiv}, vol. abs/2105.02942, 2021. [Online]. Available: \url{https://api.semanticscholar.org/CorpusID:234093560}
\BIBentrySTDinterwordspacing

\bibitem[Kim et~al.(2020)Kim, Lee, and Lee]{Kim2020UnderstandingCO}
\BIBentryALTinterwordspacing
H.~Kim, W.~Lee, and J.~Lee, ``Understanding catastrophic overfitting in single-step adversarial training,'' in \emph{AAAI Conference on Artificial Intelligence}, 2020. [Online]. Available: \url{https://api.semanticscholar.org/CorpusID:222133879}
\BIBentrySTDinterwordspacing

\bibitem[Andriushchenko and Flammarion(2020)]{Andriushchenko2020UnderstandingAI}
\BIBentryALTinterwordspacing
M.~Andriushchenko and N.~Flammarion, ``Understanding and improving fast adversarial training,'' \emph{ArXiv}, vol. abs/2007.02617, 2020. [Online]. Available: \url{https://api.semanticscholar.org/CorpusID:220363591}
\BIBentrySTDinterwordspacing

\bibitem[Golgooni et~al.(2021)Golgooni, Saberi, Eskandar, and Rohban]{Golgooni2021ZeroGradM}
\BIBentryALTinterwordspacing
Z.~Golgooni, M.~Saberi, M.~Eskandar, and M.~H. Rohban, ``Zerograd : Mitigating and explaining catastrophic overfitting in fgsm adversarial training,'' \emph{ArXiv}, vol. abs/2103.15476, 2021. [Online]. Available: \url{https://api.semanticscholar.org/CorpusID:232404666}
\BIBentrySTDinterwordspacing

\bibitem[de~Jorge et~al.(2022)de~Jorge, Bibi, Volpi, Sanyal, Torr, Rogez, and Dokania]{Jorge2022MakeSN}
\BIBentryALTinterwordspacing
P.~de~Jorge, A.~Bibi, R.~Volpi, A.~Sanyal, P.~H.~S. Torr, G.~Rogez, and P.~K. Dokania, ``Make some noise: Reliable and efficient single-step adversarial training,'' \emph{ArXiv}, vol. abs/2202.01181, 2022. [Online]. Available: \url{https://api.semanticscholar.org/CorpusID:246473010}
\BIBentrySTDinterwordspacing

\bibitem[Frankle and Carbin(2019)]{frankle2018the}
\BIBentryALTinterwordspacing
J.~Frankle and M.~Carbin, ``The lottery ticket hypothesis: Finding sparse, trainable neural networks,'' in \emph{International Conference on Learning Representations}, 2019. [Online]. Available: \url{https://openreview.net/forum?id=rJl-b3RcF7}
\BIBentrySTDinterwordspacing

\bibitem[Ramanujan et~al.(2020)Ramanujan, Wortsman, Kembhavi, Farhadi, and Rastegari]{ramanujan2020s}
V.~Ramanujan, M.~Wortsman, A.~Kembhavi, A.~Farhadi, and M.~Rastegari, ``What's hidden in a randomly weighted neural network?'' in \emph{Proceedings of the IEEE/CVF conference on computer vision and pattern recognition}, 2020, pp. 11\,893--11\,902.

\bibitem[Fu et~al.(2021)Fu, Yu, Zhang, Wu, Ouyang, Cox, and Lin]{fu2021drawing}
Y.~Fu, Q.~Yu, Y.~Zhang, S.~Wu, X.~Ouyang, D.~Cox, and Y.~Lin, ``Drawing robust scratch tickets: Subnetworks with inborn robustness are found within randomly initialized networks,'' \emph{Advances in Neural Information Processing Systems}, vol.~34, pp. 13\,059--13\,072, 2021.

\bibitem[Liu et~al.(2022)Liu, Zhao, S{\"u}sstrunk, and Salzmann]{liu2022robust}
C.~Liu, Z.~Zhao, S.~S{\"u}sstrunk, and M.~Salzmann, ``Robust binary models by pruning randomly-initialized networks,'' \emph{Advances in Neural Information Processing Systems}, vol.~35, pp. 492--506, 2022.

\bibitem[Krizhevsky et~al.(2009)Krizhevsky, Hinton, et~al.]{krizhevsky2009learning}
A.~Krizhevsky, G.~Hinton \emph{et~al.}, ``Learning multiple layers of features from tiny images,'' 2009.

\bibitem[Deng et~al.(2009)Deng, Dong, Socher, Li, Li, and Fei-Fei]{imagenet_cvpr09}
J.~Deng, W.~Dong, R.~Socher, L.-J. Li, K.~Li, and L.~Fei-Fei, ``{ImageNet: A Large-Scale Hierarchical Image Database},'' in \emph{CVPR09}, 2009.

\bibitem[Stallkamp et~al.(2012)Stallkamp, Schlipsing, Salmen, and Igel]{stallkamp2012man}
J.~Stallkamp, M.~Schlipsing, J.~Salmen, and C.~Igel, ``Man vs. computer: Benchmarking machine learning algorithms for traffic sign recognition,'' \emph{Neural networks}, vol.~32, pp. 323--332, 2012.

\bibitem[He et~al.(2016{\natexlab{a}})He, Zhang, Ren, and Sun]{he2016deep}
K.~He, X.~Zhang, S.~Ren, and J.~Sun, ``Deep residual learning for image recognition,'' in \emph{Proceedings of the IEEE conference on computer vision and pattern recognition}, 2016, pp. 770--778.

\bibitem[Zagoruyko and Komodakis(2016)]{zagoruyko2016wide}
S.~Zagoruyko and N.~Komodakis, ``Wide residual networks,'' \emph{arXiv preprint arXiv:1605.07146}, 2016.

\bibitem[Bittar et~al.(2021)Bittar, Chancelier, and Lara]{Bittar2021BestCL}
\BIBentryALTinterwordspacing
T.~Bittar, J.-P. Chancelier, and M.~D. Lara, ``Best convex lower approximations of the l 0 pseudonorm on unit balls,'' 2021. [Online]. Available: \url{https://api.semanticscholar.org/CorpusID:235254019}
\BIBentrySTDinterwordspacing

\bibitem[Varghese and Sambath(2024)]{varghese2024yolov8}
R.~Varghese and M.~Sambath, ``Yolov8: A novel object detection algorithm with enhanced performance and robustness,'' in \emph{2024 International conference on advances in data engineering and intelligent computing systems (ADICS)}.\hskip 1em plus 0.5em minus 0.4em\relax IEEE, 2024, pp. 1--6.

\bibitem[Lin et~al.(2014)Lin, Maire, Belongie, Hays, Perona, Ramanan, Doll{\'a}r, and Zitnick]{lin2014microsoft}
T.-Y. Lin, M.~Maire, S.~Belongie, J.~Hays, P.~Perona, D.~Ramanan, P.~Doll{\'a}r, and C.~L. Zitnick, ``Microsoft coco: Common objects in context,'' in \emph{European conference on computer vision}.\hskip 1em plus 0.5em minus 0.4em\relax Springer, 2014, pp. 740--755.

\bibitem[He et~al.(2016{\natexlab{b}})He, Zhang, Ren, and Sun]{he2016identity}
K.~He, X.~Zhang, S.~Ren, and J.~Sun, ``Identity mappings in deep residual networks,'' in \emph{European conference on computer vision}.\hskip 1em plus 0.5em minus 0.4em\relax Springer, 2016, pp. 630--645.

\bibitem[Dugas et~al.(2000)Dugas, Bengio, B{\'e}lisle, Nadeau, and Garcia]{dugas2000incorporating}
C.~Dugas, Y.~Bengio, F.~B{\'e}lisle, C.~Nadeau, and R.~Garcia, ``Incorporating second-order functional knowledge for better option pricing,'' \emph{Advances in neural information processing systems}, vol.~13, 2000.

\end{thebibliography}
\vfill
\newpage
\begin{wrapfigure}{l}{0.2\textwidth}
  \begin{center}
    \includegraphics[width=0.2\textwidth]{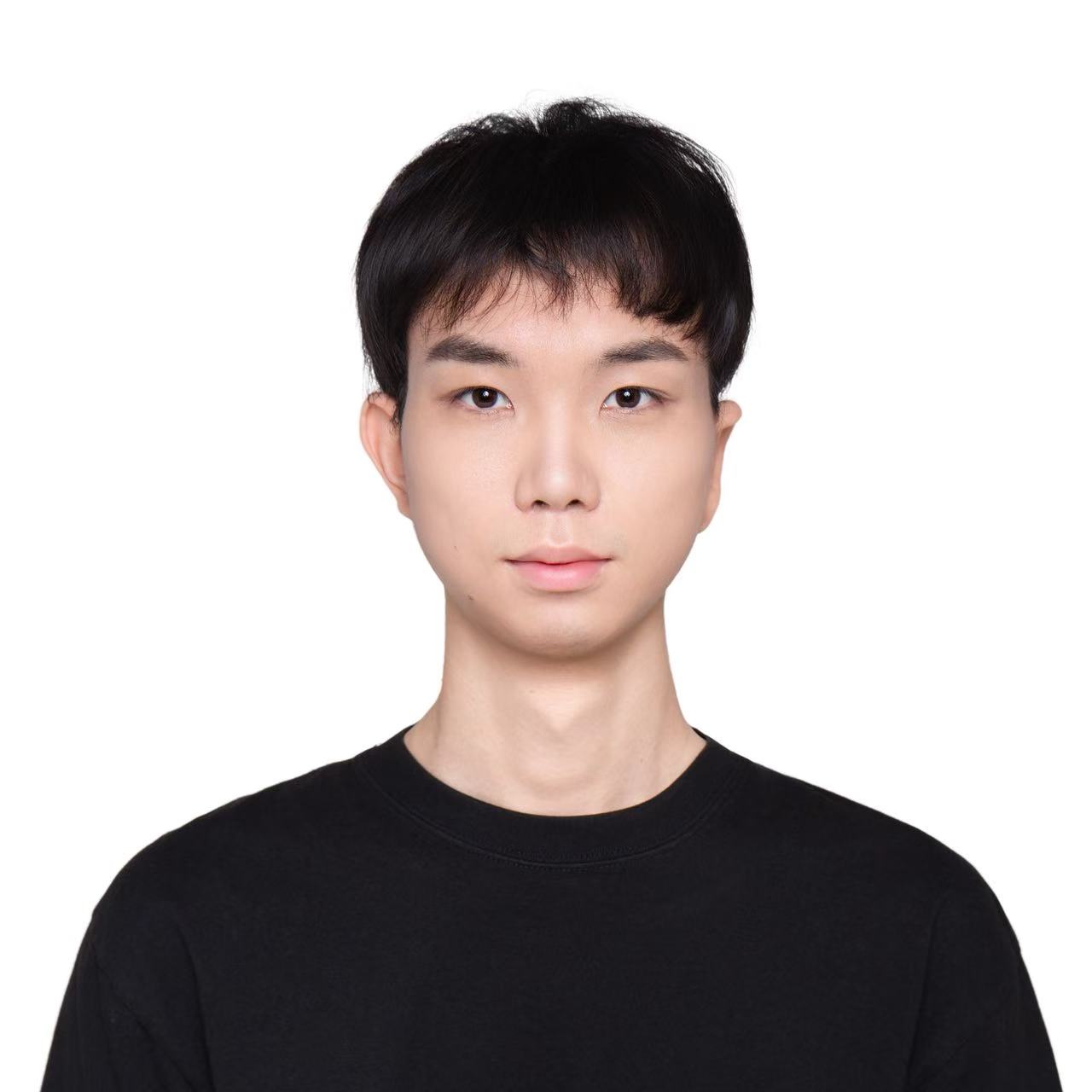}
  \end{center}
\end{wrapfigure}    
\textbf{Xuyang Zhong} is currently a Ph.D. student at the department of computer science, City University of Hong Kong and is supervised by Chen Liu. He obtained his Master's degree from Technical University of Munich in 2022 and his Bachelor's degree from Beijing Institute of Technology in 2020. His research interest mainly focuses on the optimization and robustness in deep learning.\\


\begin{wrapfigure}{l}{0.2\textwidth}
  \begin{center}
    \includegraphics[width=0.2\textwidth]{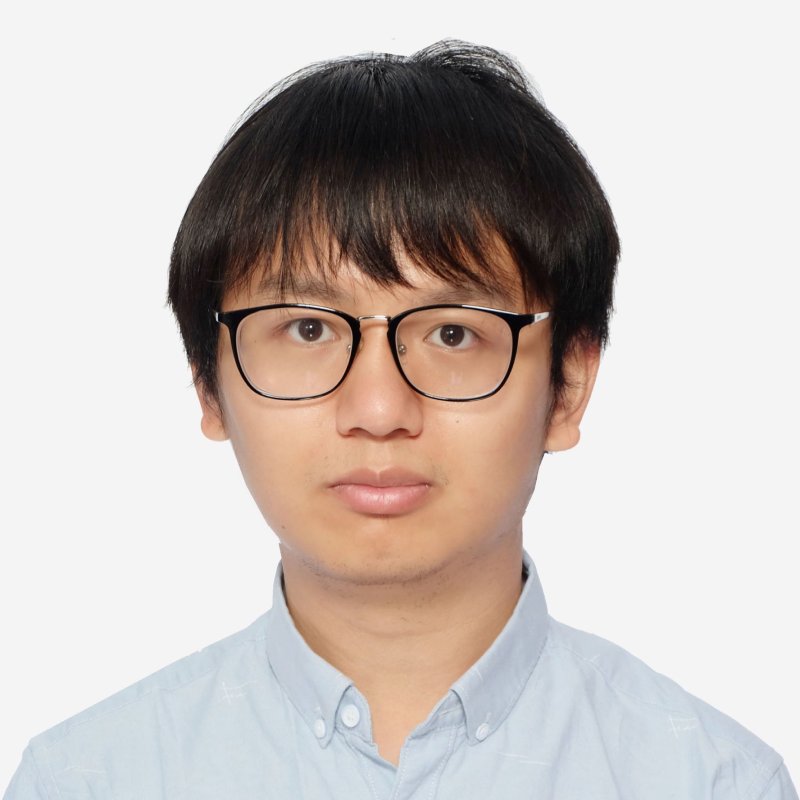}
  \end{center}
\end{wrapfigure}
\textbf{Chen Liu} is currently an assistant professor at the department of computer science, City University of Hong Kong. He obtained his Ph.D. degree from \'Ecole Polytechnique F\'ed\'erale de Lausanne (EPFL) in August 2022 and was supervised by Prof. Sabine S\"usstrunk and Dr. Mathieu Salzmann. Previously, he obtained his Master degree from \'Ecole Polytechnique F\'ed\'erale de Lausanne (EPFL) in 2017 and Bachelor degree from Tsinghua University in 2015, both in Computer Science. Chen's research interest focuses on machine learning, optimization, robustness and privacy. He was supported by Microsoft Research Ph.D Scholarship programme between 2017 and 2019.
\vfill
\newpage
\clearpage
\newpage
\section*{Sparse-PGD: A Unified Framework for Sparse Adversarial Perturbations Generation}
\appendix
\subsection{Implementation Details} \label{sec:app_imple}

 In experiments, we mainly focus on the cases of the sparsity of perturbations $\epsilon=10, 15$ and $20$, where $\epsilon=||\sum_{i=1}^{c} \bm{\delta}^{(i)}||_0$ or $||\vm||_0$, $\bm{\delta}^{(i)}\in\sR^{h\times w}$ is the $i$-th channel of perturbation $\bm{\delta}\in\sR^{h\times w \times c}$, and $\vm\in\sR^{h\times w \times 1}$ is the sparsity mask in the decomposition of $\bm{\delta} = \vp \odot \vm$, $\vp\in\sR^{h\times w \times c}$ denotes the magnitude of perturbations.
 
To exploit the strength of these attacks in reasonable running time, we run all these attacks for either $10000$ iterations or the number of iterations where their performances converge. More details are elaborated below.

\textbf{CornerSearch} \citep{croce2019sparse}: For CornerSearch, we set the hyperparameters as following: $N = 100, N_{iter} = 3000$, where $N$ is the sample size of the one-pixel perturbations, $N_{iter}$ is the number of queries. For both CIFAR-10 and CIFAR-100 datasets, we evaluate the robust accuracy on $1000$ test instances due to its prohibitively high computational complexity. 

\textbf{Sparse-RS} \citep{croce2022sparse}: For Sparse-RS, we set $\alpha_{\text{init}} = 0.8$, which controls the set of pixels changed in each iteration. Cross-entropy loss is adopted. Following the default setting in \citep{croce2022sparse}, we report the results of untargeted attacks with the maximum queries up to $10000$. 

\textbf{SparseFool} \citep{2018SparseFool}:
We apply SparseFool following the official implementation and use the default value of the sparsity parameter $\lambda = 3$. The maximum iterations per sample is set to $3000$. Finally, the perturbation generated by SparseFool is projected to the $l_0$ ball to satisfy the adversarial budget.

\textbf{PGD}$_{\bm{0}}$ \citep{croce2019sparse}: For PGD${_0}$, we include both untargeted attack and targeted attacks on the top-$9$ incorrect classes with the highest confidence scores. We set the step size to $\eta = 120000 / 255$. Contrary to the default setting, the iteration numbers of each attack increase from $20$ to $300$. Besides, $5$ restarts are adopted to boost the performance further.

\textbf{SAIF} \citep{imtiaz2022saif}: Similar to $\mathrm{PGD_0}$, we apply both untargeted attack and targeted attacks on the top-$9$ incorrect classes with $300$ iterations per attack, however, the query budget is only $100$ iterations in the original paper \citep{imtiaz2022saif}. We adopt the same $l_{\infty}$ norm constraint for the magnitude tensor $p$ as in sPGD.

\textbf{Sparse-PGD (sPGD)}: Cross-entropy loss is adopted as the loss function of both untargeted and targeted versions of our method. The small constant $\gamma$ to avoid numerical error is set to $2\times10^{-8}$. The number of iterations $T$ is $10000$ for all datasets to ensure fair comparison among attacks in Table~\ref{tab:20}. For generating $l_0$ bounded perturbations, the step size for magnitude $\vp$ is set $\alpha=0.25\times\epsilon_\infty$; the step size for continuous mask $\widetilde{\vm}$ is set $\beta=0.25\times \sqrt{h\times w}$, where $h$ and $w$ are the height and width of the input image $\vx\in\sR^{h\times w \times c}$, respectively; the tolerance for reinitialization $t$ is set to $3$. For generating structured sparse perturbations, we let $\alpha=0.0125\times\epsilon_\infty,~\beta=0.0125\times \sqrt{h\times w}$ and $t=50$.

\textbf{Sparse-AutoAttack (sAA)}: It is a cascade ensemble of five different attacks, i.e., \textbf{a)} untargeted sPGD with unprojected gradient (sPGD$_{\mathrm{u}}$), \textbf{b)} untargeted sPGD with sparse gradient (sPGD$_{\mathrm{p}}$),
and \textbf{c)} untargeted Sparse-RS. The hyper-parameters of sPGD are the same as those listed in the last paragraph. 

\textbf{Adversarial Training}: sPGD is adopted as the attack during the training phase, the number of iterations is $20$, and the backward function is randomly selected from the two different backward functions for each batch. For sTRADES, we only compute the TRADES loss when training, and generating adversarial examples is based on cross-entropy loss.
We use PreactResNet18 \citep{he2016identity} with softplus activation \citep{dugas2000incorporating} for experiments on CIFAR-10 and CIFAR-100 \cite{krizhevsky2009learning}, and ResNet34 \cite{he2016deep} for experiments on ImageNet-100 \cite{imagenet_cvpr09} and GTSRB \cite{stallkamp2012man}. We train the model for $100$ epochs on CIFAR-10 and CIFAR-100, for $40$ epochs on ImageNet-100, and for $20$ epochs on GTSRB. The training batch size is $128$ on CIFAR-10 and CIFAR-100, and $32$ on ImageNet-100 and GTSRB. The optimizer is SGD with a momentum factor of $0.9$ and weight decay factor of $5\times 10^{-4}$. The learning rate is initialized to $0.05$ and is divided by a factor of $10$ at the $\frac{1}{4}$ and the $\frac{3}{4}$ of the total epochs. The tolerance for reinitialization is set to $10$.

\subsection{More Results of Section \ref{sec:evaluation}} \label{sec:app_more_res}
In this subsection, we present the robust accuracy on CIFAR-10 and CIFAR-100 with the sparsity level $\epsilon = 10$  in Table \ref{tab:10_15}. 
The observations with different sparsity levels and on different datasets are consistent with those in Table~\ref{tab:20} and~\ref{tab:more}, which indicates the effectiveness of our method. 

\begin{table*}[ht]
\vspace{-1em}
\setlength\tabcolsep{2.5pt}
    \centering
    \caption{\small Robust accuracy of various models on different sparse attacks, where the sparsity level $\epsilon=10$. The models are trained on \textbf{CIFAR-10} and \textbf{CIFAR-100}\citep{krizhevsky2009learning}. Note that we report results of Sparse-RS (RS) with tuned hyperparameters, which outperforms its original version in \citep{croce2022sparse}. CornerSearch (CS) is evaluated on 1000 samples due to its high computational complexity.} \label{tab:10_15}
\vspace{-0.5em}
\begin{minipage}{0.48\textwidth}
    \subtable[\small \textbf{CIFAR-10}]{
    \small
    \resizebox{\textwidth}{!}{
    \small
    \begin{tabular}{l c|c|c c |c c c c c |c}
        \toprule[1.5pt]
        \multirow{2}{*}{Model} & \multirow{2}{*}{Network} & \multirow{2}{*}{Clean} & \multicolumn{2}{c|}{Black} & \multicolumn{5}{c|}{White} & \multirow{2}{*}{\textbf{sAA}} \\
          & & & CS & RS  & SF & PGD$_0$ & SAIF & \textbf{sPGD$_{\mathrm{p}}$}& \textbf{sPGD$_{\mathrm{u}}$}& \\
         \midrule[1pt]
         Vanilla & RN-18 & 93.9 &3.2 & 0.5 &  40.6 & 11.5 & 31.8 &  0.5 & 7.7 &  \textbf{0.5}\\
         
         \midrule[1pt]
         \multicolumn{10}{l}{$l_\infty$-adv. trained, $\epsilon=$ 8/255}\\
         \midrule[0.5pt]
         GD & PRN-18 & 87.4 &  36.8 & 24.5 & 69.9 & 50.3 & 63.0 & 31.0 & 37.3 & \textbf{23.2}\\
         PORT & RN-18 & 84.6 & 36.7& 27.5 & 70.7 & 46.1 & 62.6 & 31.0 &  36.2 & \textbf{25.0}\\
         DKL & WRN-28 & 92.2 & 40.9& 25.0 & 71.9 &54.2 & 64.6& 32.6 & 38.8 &  \textbf{23.7}\\
         DM& WRN-28 & 92.4 & 38.7& 23.7 &68.7& 52.7 & 62.5 & 31.2 & 37.4 &  \textbf{22.6}\\

         \midrule[1pt]
         \multicolumn{10}{l}{$l_2$-adv. trained, $\epsilon=0.5$}\\
         \midrule[0.5pt]
          HAT & PRN-18 & 90.6 & 47.3 & 40.4  &74.6 &  53.5 & 71.4 & 37.3 & 36.7 & \textbf{34.5}\\
         PORT & RN-18 & 89.8 & 46.8 &  37.7 &74.2 &  50.4 & 70.9 & 33.7 & 33.0 &  \textbf{30.6}\\
         DM & WRN-28 & 95.2 & 57.8 &  47.9 &78.3 & 65.5& 80.9& 47.1 & 48.2 & \textbf{43.1}\\
         FDA & WRN-28 & 91.8  & 55.0& 49.4 &79.6& 58.6& 77.5&  46.7 & 49.0 & \textbf{43.8}\\
         
         \midrule[1pt]
         \multicolumn{10}{l}{$l_1$-adv. trained, $\epsilon=$ 12}\\
         \midrule[0.5pt]
         $l_1$-APGD & PRN-18 & 80.7 & 51.4 & 51.1  &74.3 & 60.7 & 68.1 & 47.2 & 47.4 &  \textbf{45.9}\\
         Fast-EG-$l_1$ & PRN-18 & 76.2 & 49.7 & 48.0  &69.7 & 56.7& 63.2& 44.7 & 45.0 & \textbf{43.2}\\
         \midrule[1pt]
         \multicolumn{10}{l}{$l_0$-adv. trained, $\epsilon=$ 10}\\
         \midrule[0.5pt]
         PGD$_0$-A & PRN-18 & 85.8 & 20.7 & 16.1 & 77.1 & 66.1 & 68.7& 33.5 & 36.2 & \textbf{15.1} \\
         PGD$_0$-T & PRN-18  & 90.6 & 22.1 & 14.0 & 85.2 &  72.1& 76.6& 37.9 & 44.6& \textbf{13.9}\\
         \textbf{sAT} & PRN-18 & 86.4 & 61.0 & 57.4 & 84.1 & 82.0 & 81.1& 78.2 & 77.6& \textbf{57.4} \\
         \textbf{sTRADES} & PRN-18  & 89.8 & 74.7 & 71.6 & 88.8 & 87.7 & 86.9& 85.9 & 84.5& \textbf{71.6} \\
         \bottomrule[1.5pt]
    \end{tabular}
    } }
\end{minipage}
\begin{minipage}{0.48\textwidth}
    \subtable[\small \textbf{CIFAR-100}]{
    \small
    \resizebox{\textwidth}{!}{
    \small
    \begin{tabular}{l c|c|c c| c c c c c |c}
        \toprule[1.5pt]
        \multirow{2}{*}{Model} & \multirow{2}{*}{Network} & \multirow{2}{*}{Clean} & \multicolumn{2}{c|}{Black} & \multicolumn{5}{c|}{White} & \multirow{2}{*}{\textbf{sAA}} \\
          & & & CS & RS  & SF & PGD$_0$ & SAIF & \textbf{sPGD$_{\mathrm{p}}$}& \textbf{sPGD$_{\mathrm{u}}$}& \\
         \midrule[1pt]
         Vanilla & RN-18 & 74.3 & 1.6 & 0.3 & 20.1 & 1.9 & 9.0 & 0.1 & 0.9 &   \textbf{0.1}\\
         \midrule[1pt]
         \multicolumn{10}{l}{$l_\infty$-adv. trained, $\epsilon=$ 8/255}\\
         \midrule[0.5pt]
         HAT & PRN-18 & 61.5 & 12.6 & 9.3 & 39.1 & 19.1 &26.8 & 11.6 & 14.2 & \textbf{8.5}\\
         FDA & PRN-18 & 56.9 & 16.3 & 12.3 & 42.2 & 23.0 & 30.7 & 14.9 & 17.8 & \textbf{11.6}\\
         DKL & WRN-28 & 73.8 & 12.4 & 6.3 & 44.9 & 20.9& 26.5 &  10.5 & 14.0 &  \textbf{6.1 }\\
         DM& WRN-28 & 72.6  &14.0 & 8.2  & 46.2 & 23.4 & 29.8 & 12.7 & 15.8 & \textbf{8.0}\\
         
         \midrule[1pt]
         \multicolumn{10}{l}{$l_1$-adv. trained, $\epsilon=$ 6}\\
         \midrule[0.5pt]
         $l_1$-APGD & PRN-18 & 63.2  & 22.7 & 22.1 &47.7 & 33.0 & 43.5 & 19.7 & 20.3 & \textbf{18.5}\\
         Fast-EG-$l_1$ & PRN-18 &59.4  &21.5 & 21.0 &44.8&30.6 & 39.5& 18.9 & 18.6& \textbf{17.3}\\
         \midrule[1pt]
         \multicolumn{10}{l}{$l_0$-adv. trained, $\epsilon=$ 10}\\
         \midrule[0.5pt]
         PGD$_0$-A & PRN-18 & 66.1 & 9.3 & 7.1 & 57.9 & 29.9 & 39.5& 13.9& 20.4 & \textbf{6.5}\\
         PGD$_0$-T & PRN-18 & 70.7 & 14.8 & 10.5 & 63.5&  46.3& 51.7& 24.5& 28.6 & \textbf{10.2}\\
         \textbf{sAT} & PRN-18 & 67.0 & 44.3 & 41.6 & 65.9 & 61.6 & 60.9& 56.8& 58.0 & \textbf{41.6}\\
         \textbf{sTRADES} & PRN-18 & 70.9 & 52.8 & 50.3 & 69.2 & 67.2 & 65.2 & 64.0& 63.7& \textbf{50.2}\\
         \bottomrule[1.5pt]
    \end{tabular}
    }}
\end{minipage}
\end{table*}

\subsection{Ablation Studies} \label{sec:ab}
We conduct ablation studies in this section. We focus on CIFAR-10 and the sparsity level $\epsilon = 20$. Unless specified, we use the same configurations as in Table~\ref{tab:20}.

\begin{table}[htb]
    \vspace{-1em}
    \centering
    \caption{\small Ablation study of each component in sPGD$_{\mathrm{proj}}$ in terms of robust accuracy. The model is Fast-EG-$l_1$ trained on CIFAR-10.
    }
    \small
    \begin{tabular}{l | c}
        \toprule[1.5pt]
        Ablations & Robust Acc.\\
        \midrule[0.5pt]
        Baseline (PGD$_0$ w/o restart) & 49.4\\
         \quad\quad + Decomposition: $\bm{\delta} = \vp \odot \vm$ & 58.0 (\red{+8.6})\\
         \quad\quad + Continuous mask $\widetilde{\vm}$& 33.9  (\green{-15.5})\\
         \quad\quad + Random reinitialization & \textbf{18.1 (\green{-31.3})}\\
         \bottomrule[1.5pt]
    \end{tabular}
    \label{tab:ab_pgd}
\end{table}

\textbf{sPGD:}
We first validate the effectiveness of each component in sPGD. The result is reported in Table \ref{tab:ab_pgd}. We observe that naively decomposing the perturbation $\bm{\delta}$ by $\bm{\delta} = \vp \odot \vm$ and updating them separately can deteriorate the performance.
By contrast, the performance significantly improves when we update the mask $\vm$ by its continuous alternative $\widetilde{\vm}$ and $l_0$ ball projection. This indicates that introducing $\widetilde{\vm}$ greatly mitigates the challenges in optimizing discrete variables.
Moreover, the results in Table~\ref{tab:ab_pgd} indicate the performance can be further improved by the random reinitialization mechanism, which encourages exploration and avoids trapping in a local optimum.

In addition, we compare the performance when we use different step sizes for the magnitude tensor $\vp$ and the sparsity mask $\vm$. As shown in Table \ref{tab:ab_alpha} and \ref{tab:ab_beta}, the robust accuracy does not vary significantly with different step sizes. It indicates the satisfying robustness of our method to different hyperparameter choices. In practice, We set $\alpha$ and $\beta$ to $0.25$ and $0.25\times\sqrt{hw}$, respectively. Note that $h$ and $w$ denote the height and width of the image, respectively, which are both $32$ in CIFAR-10.
\begin{table}[ht]
    \vspace{-1em}
    \caption{\small Robust accuracy at different step sizes $\alpha$ for magnitude $p$. The evaluated attack is sPGD$_{\mathrm{proj}}$. The model is Fast-EG-$l_1$ trained on CIFAR-10.}
    \centering
    \small
    \label{tab:ab_alpha}
    \begin{tabular}{c c c c c c c}
        \toprule[1.5pt]
        $\alpha$ & $\frac{1}{16}$ & $\frac{1}{8}$ & $\frac{1}{4}$ & $\frac{1}{2}$ & $\frac{3}{4}$ & $1$\\
        \midrule[1pt]
        Acc. & 19.6 &18.8  & \textbf{18.1} & 18.4
        & 18.7 & 19.0\\
        \bottomrule[1.5pt]
        \vspace{-2em}
    \end{tabular}
\end{table}

\begin{table}[ht]
    \vspace{-1em}
    \centering
    \makeatletter\def\@captype{table}\makeatother \caption{\small Robust accuracy at different step sizes $\beta$ for mask $m$. The evaluated attack is sPGD$_{\mathrm{proj}}$. The model is Fast-EG-$l_1$ trained on CIFAR-10. Note that $h$ and $w$ are both $32$ in CIFAR-10.
    }
    \small
    \label{tab:ab_beta}
    \begin{tabular}{c c c c c c c}
        \toprule[1.5pt]
        $\beta/\sqrt{hw}$ & $\frac{1}{16}$ & $\frac{1}{8}$ & $\frac{1}{4}$ & $\frac{1}{2}$ & $\frac{3}{4}$ & $1$\\
        \midrule[1pt]
        Acc. & 20.0&  19.3& \textbf{18.1} & 18.4 & 19.4& 21.0\\
        \bottomrule[1.5pt]
    \end{tabular}
\end{table}

Ultimately, we compare the performance of sPGD with different tolerance iterations for reinitialization. As shown in Table \ref{tab:ab_tol}, the performance of our method remains virtually unchanged, which showcases that our approach is robust to different choices of tolerance for reinitialization.
\begin{table}[ht]
    \vspace{-1em}
    \centering
        \makeatletter\def\@captype{table}\makeatother \caption{\small Robust accuracy at different tolerance for reinitialization $t$ during attacking. The evaluated attack is sPGD$_{\mathrm{proj}}$. The model is Fast-EG-$l_1$ trained on CIFAR-10.}
        \label{tab:ab_tol}
        \small
        \begin{tabular}{c|c c c c c}
            \toprule[1.5pt]
            $t$ & 1& 3 & 5 & 7 & 10 \\
            \midrule[1pt]
            Acc. & 18.1 & \textbf{18.1} & 18.5 & 18.5& 18.5 \\
            \bottomrule[1.5pt]
        \end{tabular}
        \vspace{-1em}
\end{table}

\textbf{Adversarial training:}
We conduct preliminary exploration on adversarial training against sparse perturbations, since sPGD can be incorporated into any adversarial training variant.
Table~\ref{tab:20},~\ref{tab:more} and \ref{tab:10_15} study sAT and sTRADES, while sTRADES outperforms sAT in all cases. In addition, the training of sAT is relatively unstable in practice, so we focus on sTRADES for ablation studies in this section.
We leave the design of sPGD-adapted adversarial training variants to further improve model robustness as a future work.

Table~\ref{tab:ab_adv2} demonstrates the performance when we use different backward functions. The policies include always using the sparse gradient (Proj.), always using the unprojected gradient (Unproj.), alternatively using both backward functions every $5$ epochs (Alter.) and randomly selecting backward functions (Rand.). The results indicate that randomly selecting backward functions has the best performance.
In addition, Table~\ref{tab:ab_adv1} demonstrates the robust accuracy of models trained by sTRADES with different multi-$\epsilon$ strategies. The results indicate that multi-$\epsilon$ strategy helps boost the performance.
The best robust accuracy is obtained when the adversarial budget for training is $6$ times larger than that for test.
Furthermore, we also study the impact of different tolerances during adversarial training in Table \ref{tab:ab_adv3}. The results show that higher tolerance during adversarial training benefits the robustness of the model, and the performance reaches its best when the tolerance is set to $10$.
\begin{table}[t]
    \vspace{-1em}
    \centering
    \caption{\small Ablation study on different policies during adversarial training. The model is PRN18 trained by sTRADES with $6\times k$. The robust accuracy is obtained through sAA.
    }
    \label{tab:ab_adv2}
    \small
    \begin{tabular}{c| c c c c}
        \toprule[1.5pt]
       Policy & Proj. & Unproj. & Alter. & Rand.\\
         \midrule[1pt]
         Acc. & 41.5 & 39.4 & 51.5 & \textbf{61.7}\\
         \bottomrule[1.5pt]
    \end{tabular}
    \vspace{-0.5em}
\end{table}
\begin{table}[t]
    \centering
    \caption{\small Ablation study on multi-$\epsilon$ strategy during adversarial training. The model is PRN18 trained by sTRADES with random policy. The robust accuracy is obtained through sAA.}
    \small
    \label{tab:ab_adv1}
    \begin{tabular}{c| c c c c c c}
        \toprule[1.5pt]
        $\epsilon$ & $1\times$ & $2\times$ & $4\times$ & $6\times$& $8\times$ & $10\times$\\
        \midrule[1pt]
        Acc. & 34.0 & 39.7 & 54.5 & \textbf{61.7} & 60.2 & 55.5\\
        \bottomrule[1.5pt]
    \end{tabular}
    \vspace{-0.5em}
\end{table}
\begin{table}[!h]
    \centering
    \makeatletter\def\@captype{table}\makeatother \caption{\small Ablation study on tolerance for reinitialization $t$ during adversarial training. The model is PRN18 trained by sTRADES with 6 $\times k$ and tolerance $t=$ 3. The robust accuracy is obtained through sAA.}
    \label{tab:ab_adv3}
    \small
    \begin{tabular}{c|c c c}
        \toprule[1.5pt]
        $t$ & 3 & 10 & 20 \\
        \midrule[1pt]
        Acc. & 51.7 & \textbf{61.7} & 60.0\\
        \bottomrule[1.5pt]
    \end{tabular}
    \vspace{-1em}
\end{table}

\textbf{Structured Sparse Perturbations:}
Since the optimal hyperparameters for structured sparse perturbations differ from those for unstructured $l_0$ bounded perturbations, we conduct further ablation studies in the structured sparse setting.

The results in Table \ref{tab:ab_alpha_struct}, \ref{tab:ab_beta_struct} and \ref{tab:ab_tol_struct} suggest that the step size $\alpha$, $\beta$ and the tolerance $t$ for generating structured sparse perturbations should be smaller than those for unstructured cases. This can be attributed to the requirement for more exploitable attacks when generating structured sparse perturbations, similar to conventional $l_\infty$ and $l_2$ attacks \citep{croce2020reliable}.

The multi-$\epsilon$ strategy has been proven effective in $l_0$ adversarial training. However, in the context of structured sparse attack, particularly in the patch setting, we can increase either the sparsity level $\epsilon$ or the patch size to increase the corresponding $l_0$ adversarial budget. Therefore, we compare two different strategies here, i.e., multi-$\epsilon$ strategy and multi-size strategy. Specifically, the multi-$\epsilon$ strategy increases the sparsity level $\epsilon$ while maintaining the patch size unchanged during training, but the multi-size strategy only increases the patch size to make the $l_0$ sparsity level of generated perturbations similar to that in the multi-$\epsilon$ strategy. The results reported in Table \ref{tab:ab_struct_adv} indicate that the multi-$\epsilon$ strategy provides better performance.

\begin{table}[!h]
    \vspace{-0.5em}
    \caption{\small Robust accuracy at different step sizes $\alpha$ for magnitude $p$. The evaluated attack is sPGD$_{\mathrm{proj}}$ generating $5\times5$ patch ($\epsilon=$ 1). The model is sTRADES-$l_0$trained on CIFAR-10.
    }
    \centering
    \small
    \label{tab:ab_alpha_struct}
    \begin{tabular}{c c c c c c c}
        \toprule[1.5pt]
        $\alpha$  & $0.01$ & $0.0125$ & $0.05$ & $0.1$ & $0.15$ & $0.25$\\
        \midrule[1pt]
       Acc. &28.1 &  \textbf{26.7}& 33.8 & 41.5 & 47.9& 58.1\\
        \bottomrule[1.5pt]
    \end{tabular}
\end{table}

\begin{table}[!h]
    \vspace{-1em}
    \centering
    \makeatletter\def\@captype{table}\makeatother \caption{\small Robust accuracy at different step sizes $\beta$ for mask $m$. The evaluated attack is sPGD$_{\mathrm{p}}$ generating 5$\times$5 patch ($\epsilon=$ 1). The model is sTRADES-$l_0$ trained on CIFAR-10. Note that $h$ and $w$ are both $32$ in CIFAR-10.
    }
    \small
    \label{tab:ab_beta_struct}
    \begin{tabular}{c c c c c c c}
        \toprule[1.5pt]
        $\beta/\sqrt{hw}$ & $0.01$ & $0.0125$ & $0.05$ & $0.1$ & $0.15$ & $0.25$\\
        \midrule[1pt]
        Acc. & 27.4&  \textbf{26.7}& 34.1 & 37.5 & 38.3& 35.1\\
        \bottomrule[1.5pt]
    \end{tabular}
\end{table}

\begin{table}[!h]
    \vspace{-1em}
    \centering
        \makeatletter\def\@captype{table}\makeatother \caption{\small Robust accuracy at different tolerance for reinitialization $t$ during attacking. The evaluated attack is sPGD$_{\mathrm{p}}$ generating 5$\times$5 patch ($\epsilon=$ 1). The model is sTRADES-$l_0$ trained on CIFAR-10.}
        \label{tab:ab_tol_struct}
        \small
        \begin{tabular}{c|c c c c c}
            \toprule[1.5pt]
            $t$ & 3& 10 & 50 & 100 & 150 \\
            \midrule[1pt]
            Acc. & 36.1 & 28.7 & \textbf{26.7} &  28.1 & 33.7 \\
            \bottomrule[1.5pt]
        \end{tabular}
\end{table}

\begin{table}[!h]
    \centering
    \caption{\small Ablation study on multi-$\epsilon$ / multi-size strategy during adversarial training against 3$\times$3 patch ($\epsilon=2$). The model is PRN18 trained by sTRADES with random policy.}
    \small
    \label{tab:ab_struct_adv}
    \begin{tabular}{c| c c c c c}
        \toprule[1.5pt]
       Attack & RS & LOAP & sPGD$_{\mathrm{p}}$& sPGD$_{\mathrm{u}}$& sAA\\
        \midrule[1pt]
        Multi-$\epsilon$ & \textbf{77.5} & \textbf{81.4} & \textbf{72.8} & \textbf{81.7} & \textbf{72.3}\\
        Multi-size & 75.5 & 79.2 & 68.4 & 79.8 & 67.9\\
        \bottomrule[1.5pt]
    \end{tabular}
\end{table}

\subsection{More Visualization} \label{sec:app_visual}

More adversarial examples are shown in Figure \ref{fig:visual_app}. Additionally, we showcase some $l_0$ bounded adversarial examples generated on our sAT model in Figure \ref{fig:visual_adv}. Compared to Figure \ref{fig:visual} (a)-(b) and Figure \ref{fig:visual_app} (a)-(b), the perturbations generated on the adversarially model are mostly located in the foreground of images. It is consistent with the intuition that the foreground of an image contains most of the semantic information \cite{imtiaz2022saif}.

\begin{figure*}[ht]
    \centering
    \subfigure[$l_0$, $\epsilon=100$]{\includegraphics[width=0.21\textwidth]{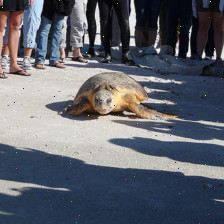}}
    \subfigure[$l_0$, $\epsilon=100$]{\includegraphics[width=0.21\textwidth]{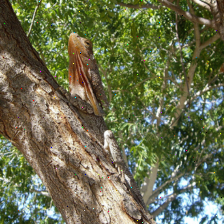}}
    \subfigure[column, $\epsilon=1$]{\includegraphics[width=0.21\textwidth]{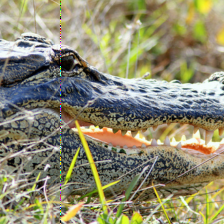}}
    \subfigure[column, $\epsilon=1$]{\includegraphics[width=0.21\textwidth]{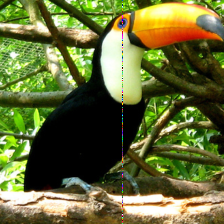}}
    \\
    \subfigure[patch, $14\times14$, $\epsilon=1$]{\includegraphics[width=0.21\textwidth]{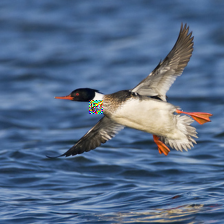}}
    \subfigure[patch, $14\times14$, $\epsilon=1$]{\includegraphics[width=0.21\textwidth]{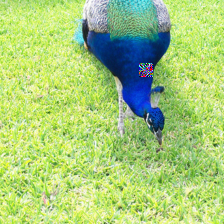}}
    \subfigure[patch, $10\times10$, $\epsilon=2$]{\includegraphics[width=0.21\textwidth]{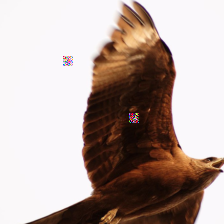}}
    \subfigure[patch, $10\times10$, $\epsilon=2$]{\includegraphics[width=0.21\textwidth]{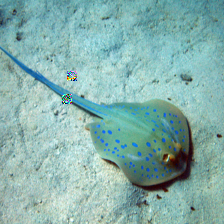}}
    \\
    \subfigure[heart, $20\times20$, $\epsilon=1$]{\includegraphics[width=0.21\textwidth]{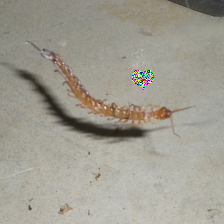}}
    \subfigure[heart, $20\times20$, $\epsilon=1$]{\includegraphics[width=0.21\textwidth]{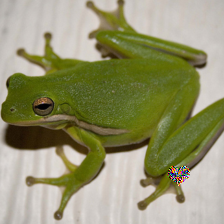}}
    \subfigure[circle, $60\times60$, $\epsilon_\infty=64/255$]{\includegraphics[width=0.21\textwidth]{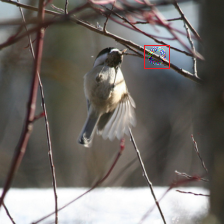}}
    \subfigure[circle, $60\times60$, $\epsilon_\infty=64/255$]{\includegraphics[width=0.21\textwidth]{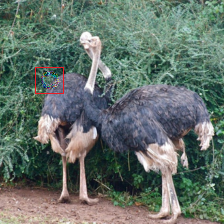}}
    \caption{\small Visualization of different sparse adversarial examples in ImageNet-100. The model is vanilla ResNet-34. The predictions before (left) and after (right) attack are (a) loggerhead$\rightarrow$stingray, (b) frilled lizard$\rightarrow$hornbill, (c) American alligator$\rightarrow$African crocodile, (d) toucan$\rightarrow$hornbill, (e) red-breasted merganser$\rightarrow$drake, (f)peacock$\rightarrow$cock, (g) kite$\rightarrow$vulture, (h) stingray$\rightarrow$electric ray, (i) centipede$\rightarrow$stingray, (j) tree frog$\rightarrow$tailed frog, (k) chickadee$\rightarrow$junco, and (l) ostrich$\rightarrow$peacock. \textbf{Note that the red squares in (k) and (l) are just for highlighting the perturbation position and not part of perturbations.}}
    \label{fig:visual_app}
\end{figure*}
\begin{figure*}[ht]
    \centering
    \subfigure[$l_0$, $\epsilon=200$]{\includegraphics[width=0.21\textwidth]{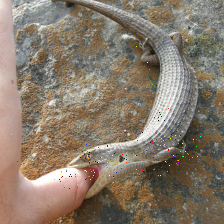}}
    \subfigure[$l_0$, $\epsilon=200$]{\includegraphics[width=0.21\textwidth]{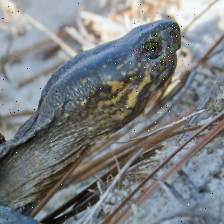}}
    \subfigure[$l_0$, $\epsilon=200$]{\includegraphics[width=0.21\textwidth]{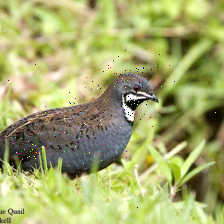}}
    \subfigure[$l_0$, $\epsilon=200$]{\includegraphics[width=0.21\textwidth]{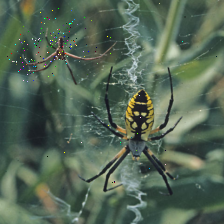}}
    \caption{\small Visualization of successful sparse adversarial examples in ImageNet-100. The model is ResNet-34 trained with sAT-$l_0$. The predictions before (left) and after (right) attack are (a) alligator lizard$\rightarrow$water snake, (b) mud turtle$\rightarrow$box turtle, (c) quail$\rightarrow$partridge, and (d) black and gold garden spider$\rightarrow$garden spider.}
    \label{fig:visual_adv}
\end{figure*}
\end{document}